\documentclass{article}
\usepackage{ijcai17}
\usepackage{mathrsfs}

\usepackage{amsmath}
\usepackage{algorithm}

\usepackage[noend]{algpseudocode}
\usepackage{times}
\usepackage{float}

\usepackage[table]{xcolor}
\usepackage{stmaryrd}
\usepackage{amssymb}

\usepackage{pgfplots}
\usepackage{subfig}

\usepackage{tikz}

\usepackage{color}

\newcommand\cactus[4]{%
\begin{tikzpicture}[scale=.75]
	\begin{axis}[scatter/classes={#3},
		mark size=1.5pt,
		ylabel=CPU Time,
		xlabel=#1,
		legend style={ at={(0.02,.98)}, font=\scriptsize, anchor=north west},
		]

		\foreach \bidule in #4 {

		\addplot[scatter, scatter src=explicit symbolic] coordinates {\bidule};

		}
		\legend{#2}
	\end{axis}
\end{tikzpicture}
}

\title{Finding Robust Solutions to Stable Marriage}

\author{
Begum Genc\textsuperscript{1}, 
Mohamed Siala\textsuperscript{1},  
Gilles Simonin\textsuperscript{2},
Barry O'Sullivan\textsuperscript{1}\\
 \textsuperscript{1} Insight, Centre for Data Analytics, Department of Computer Science, University College Cork, Ireland\\
 \textsuperscript{2} TASC, Institut Mines Telecom Atlantique, LS2N UMR 6004, Nantes, France\\
 \{begum.genc, mohamed.siala, barry.osullivan\}@insight-centre.org, gilles.simonin@imt-atlantique.fr
}

\usepackage{amsthm}
\usepackage{color}

\newtheorem{theorem}{Theorem}
\newtheorem{lemma}{Lemma}

\newtheorem{definition}{Definition}

\newtheorem{corollary}{Corollary}

\def\man#1{m_{#1}}
\def\woman#1{w_{#1}}

\def\pair#1#2{\langle \man{#1}, \woman{#2} \rangle}
\def\noLBLpair#1#2{\langle {#1}, {#2} \rangle}
\def\x#1#2{x_{{{#1}},{{#2}}}} 

\def\mbest{M_0}
\def\wbest{M_z}

\def\ro#1{\rho_{#1}}

\def\smup#1{M_{\textsc{UP}}^{*#1}}
\def\smdown#1{M_{\textsc{DOWN}}^{*#1}}

\def\sManUp#1{S_{\textsc{UP}}^{*#1}}
\def\sManDown#1{S_{\textsc{DOWN}}^{*#1}}

\def\sm#1{M_{#1}}

\def\s#1{s_{#1}}

\def\y#1#2{y_{#1}^{#2}}
\def\z#1#2{z_{#1}^{#2}}
\def\sUp#1#2{s_{up_{#1}}^{#2}}
\def\sDown#1#2{s_{down_{#1}}^{#2}}
\def\dup#1{d_{up}^{#1}}
\def\ddown#1{d_{down}^{#1}}

\def\roset{{\cal{V}}}

\def\preT#1{{{N}_t^-}(#1)}
\def\sucT#1{{{N}_t^+}(#1)}

\def\preNT#1{{{N}^-}(#1)}
\def\sucNT#1{{{N}^+}(#1)}

\def\roelt#1#2{\rho_{e_{{#1},{#2}}}}
\def\roprd#1#2{\rho_{p_{{#1},{#2}}}}

\def\elt#1#2{{e_{{#1},{#2}}}}
\def\prd#1#2{{p_{{#1},{#2}}}}

\begin{document}

\maketitle

\begin{abstract}

We study the notion of robustness in stable matching problems. 
We first define robustness by introducing \textit{(a,b)-supermatches}.
An $(a,b)$-supermatch is a stable matching in which if any $a$ pairs break up it is possible to find another stable matching by changing the partners of those $a$ pairs and the partners of at most $b$ other pairs.
In this context, we define the most robust stable matching as a 
$(1,b)$-supermatch where b is minimum. 
We first show that checking whether a given stable matching is a $(1,b)$-supermatch can be done in polynomial time. 
Next, we use this procedure to design a constraint programming model, a local search approach, and a genetic algorithm to find the most robust stable matching. 
Our empirical evaluation on large instances shows that local search outperforms the other approaches.

%

\end{abstract}

\section{Introduction}

Heraclitus, the Greek philosopher is quoted as saying that ``Change is the only constant". 
Therefore, it is essential to build robust systems that can be repaired by only minor changes in case of an unforeseen event~\cite{Sussman07buildingrobust}.
Although it is usually difficult to provide robustness to a complex problem, as it may be computationally expensive, a robust solution reduces the cost of future repairs.

This paper focuses on matching problems under preferences, where the aim is to find an assignment between two disjoint sets of agents while respecting an optimality criterion. 
Each agent has an ordinal preference ranking over agents of the other set. 
These types of problems have been widely studied by different research communities such as computer scientists and economists over the years; in fact, the 2012 Nobel Prize for Economics was awarded to Shapley and Roth for their work on stable allocations.
Some of the variants can be listed as assigning residents to hospitals (HR), matching men and women to find stable marriages (SM)~\cite{GaleShapley1962CollegeAdmStabilityMarriage,Gusfield:1989:SMP:68392}, and finding donors for kidney patients~\cite{RePEc:wpa:wuwpga:0408001}.

Stable Marriage (SM)~\cite{GaleShapley1962CollegeAdmStabilityMarriage} is the first and the most studied variant of these problems. In SM, the sets of agents correspond to men and woman. The goal is to find a matching $M$ between men and women where each person is matched to at most one partner from the opposite sex such that there is no man and woman that prefers each other to their situations in $M$. 
Such a matching is called \textit{stable}.
We primarily work on the Stable Marriage problem, but the problem is also meaningful in the context of other matching problem variants.

We introduce the notion of $(a,b)$-supermatches as a measure of robustness for SM. 
Informally, a stable matching $M$ is called an $(a,b)$-supermatch if any $a$ agents decide to break their matches in $M$, thereby breaking $a$ pairs, it is possible to ``repair'' $M$ (i.e., find another stable matching) by changing the assignments of those $a$ agents and the assignments of at most $b$ others. 
This concept is inspired by the notion of $(a,b)$-supermodels in boolean 
 satisfiability~\cite{Ginsberg98supermodelsand} and super solutions in constraint programming~\cite{04-hebrard1,DBLP:conf/cpaior/HebrardHW04,07-hebrard-phd}.
Note that, if one agent wants to break his/her current match,
 the match of the partner breaks also. 
Therefore, when we mention $a$ (or $b$) as the number of agents, we always refer to the agents from the same set.

There exist different definitions for robustness that leads researchers to ambiguity~\cite{DBLP:journals/constraints/Climent15}.
We take the definitions of supermodels and super solutions as reference and use the term \textit{robust} for characterising repairable stable matchings ~\cite{Ginsberg98supermodelsand,DBLP:conf/cpaior/HebrardHW04}.

In order to give additional insight into the problem, we motivate robustness on the Hospital/Residents (HR) problem. 
The HR problem is a one-to-many generalization of SM.
In HR, each hospital has a capacity and a preference list in which they rank the residents. 
Similarly, each resident has a preference list over the hospitals. 
A $(1,b)$-supermatch means that if a resident wants to leave his assigned hospital or a hospital does not want to have one of its current residents, it is possible to move that resident to another hospital by also moving at most $b$ other residents to other hospitals. 
By minimising $b$, we can ensure that the required number of additional relocations to provide a repair is minimal and therefore the solution is robust.
In practice, the most robust matching minimises the cost for recovering from unwanted and unforeseen events.

The first contribution of this paper is a polynomial time procedure to verify whether a given stable matching is a $(1,b)$-supermatch. Next, based on this procedure, we design a constraint programming (CP) model, as well as local search (LS) and genetic algorithm (GA) to find the most robust stable matching.
Last, we give empirical evidence that the local search algorithm is by far the most efficient approach to tackle this problem. 

The structure of the paper is as follows: in Section~\ref{sec:bgAndNotations}, the basics of the Stable Marriage problem and the different notations are introduced. 
In Section~\ref{sec:verification}, our polynomial-time method is proposed.
Next, we give the CP model in Section~\ref{sec:CP} and the two meta-heuristic algorithms in Section~\ref{sec:Metaheuristics}.
Finally, we give our experimental study in Section~\ref{sec:Experiments}.

\section{Stable Marriage}
\label{sec:bgAndNotations}

The Stable Marriage problem takes as input a set of men $U = \{\man{1}, \man{2},\ldots, \man{n_1} \}$ and a set of woman $W = \{\woman{1}, \woman{2},\ldots, \woman{n_2} \}$ where each person has an ordinal preference list over people of the opposite sex. 
For the sake of simplicity we suppose in the rest of the paper that $n_1=n_2$ (denoted by n), and that each person expresses a complete preference ranking over the set of opposite sex. 
In the rest of the paper we interchangeably use $i$ to denote man $\man{i}$ in $U$ or similarly $j$ to denote woman $\woman{j}$ in $W$.

A \textit{matching} $M$ is a one-to-one correspondence between $U$ and $W$. 
For each man $\man{i}$, $M(\man{i})=\woman{j}$ is called the partner of $\man{i}$ in matching $M$. 
In the latter case, we denote by $M(\woman{j})=\man{i}$.
We shall sometimes abuse notation by considering $M$ as a set of pairs. 
In that case, a pair $\pair{i}{j} \in M$ iff $M(\man{i})=\woman{j}$.  
A pair $\pair{i}{j}$ is said to be \textit{blocking} a matching $M$
 if $\man{i}$ prefers $\woman{j}$ to $M(\man{i})$ and $\woman{j}$ prefers $\man{i}$ to $M(\woman{j})$. 
A \textit{matching} $M$ is called \textit{stable} if there exists no blocking pairs for $M$.  
A \textit{pair} $\pair{i}{j}$ is said to be \textit{stable} if it appears in a stable matching. 
Also, a pair $\pair{i}{j}$ is \textit{fixed} if $\pair{i}{j}$ appears in every stable matching.


The structure that represents all stable matchings forms a \textit{lattice} $\mathscr{M}$.
In this lattice, the man-optimal matching is denoted by $\mbest$ and the woman-optimal (man-pessimal) matching is denoted by $\wbest$.
A stable matching $\sm{i}$ \textit{dominates} a stable matching $\sm{j}$, denoted by $\sm{i} \preceq \sm{j}$, if every man prefer their matches in $\sm{i}$ to $\sm{j}$ or indifferent between them.
The size of a lattice can be exponential as the number of all stable matchings can be exponential~\cite{Irving:1986:CCS:14821.14824}. Therefore, making use of this structure for measuring the robustness of a stable matching is not practical.

Table~\ref{table:sm} gives an example of a Stable Marriage instance with $7$ men/women. 
For the sake of clarity, we denote each man $m_i$ with $i$ and each woman $w_j$ with $j$.


\begin{table}[ht]
\centering
\begin{tabular}{|l|l|l|l|l|l|l|l|l|l|l|l|l|l|l|l|l|}
\cline{1-8} \cline{10-17}
$m_0$ & \multicolumn{7}{l|}{0 6 5 2 4 1 3} &  & $w_0$ & \multicolumn{7}{l|}{2 1 6 4 5 3 0} \\ \cline{1-8} \cline{10-17} 
$m_1$ & \multicolumn{7}{l|}{6 1 4 5 0 2 3} &  & $w_1$ & \multicolumn{7}{l|}{0 4 3 5 2 6 1} \\ \cline{1-8} \cline{10-17} 
$m_2$ & \multicolumn{7}{l|}{6 0 3 1 5 4 2} &  & $w_2$ & \multicolumn{7}{l|}{2 5 0 4 3 1 6} \\ \cline{1-8} \cline{10-17} 
$m_3$ & \multicolumn{7}{l|}{3 2 0 1 4 6 5} &  & $w_3$ & \multicolumn{7}{l|}{6 1 2 3 4 0 5} \\ \cline{1-8} \cline{10-17} 
$m_4$ & \multicolumn{7}{l|}{1 2 0 3 4 5 6} &  & $w_4$ & \multicolumn{7}{l|}{4 6 0 5 3 1 2} \\ \cline{1-8} \cline{10-17} 
$m_5$ & \multicolumn{7}{l|}{6 1 0 3 5 4 2} &  & $w_5$ & \multicolumn{7}{l|}{3 1 2 6 5 4 0} \\ \cline{1-8} \cline{10-17} 
$m_6$ & \multicolumn{7}{l|}{2 5 0 6 4 3 1} &  & $w_6$ & \multicolumn{7}{l|}{4 6 2 1 3 0 5} \\ \cline{1-8} \cline{10-17} 
\end{tabular}

\caption{Preference lists for men (left) and women (right) for a sample Stable Marriage instance of size 7.}
\label{table:sm}
\end{table}

Let $M$ be a stable matching. 
A \textit{rotation} $\ro{} = (\pair{i_0}{j_0}, \pair{i_1}{j_1}, \ldots, \pair{i_{l-1}}{j_{l-1}})$ (where $l \in \mathbb{N}^*$) is an ordered list of pairs in $M$ such that changing the partner of each man $\man{i_c}$ to the partner of the next man $\man{i_{c+1}}$, where $0 \leq c \leq l-1$ and the operation +1 is modulo $l$, in the list $\ro{}$ leads to a stable matching $M / \ro{}$.
The latter is said to be obtained after \textit{eliminating} $\ro{}$ from $M$.
In this case, we say that $\pair{l_i}{l_i}$ is eliminated by $\ro{}$, $\pair{l_i}{l_{i+1}}$ is produced
 by $\ro{}$, and that $\ro{}$ is \textit{exposed} in $M$. 

For each pair $\pair{i}{j} \in \ro{}$, we say $\man{i}$ (or $\woman{j}$) is \textit{involved} in $\ro{}$.
Additionally, for each $\pair{i}{j}$ such that $\pair{i}{j} \notin \in \mbest$, 
there exists a unique rotation $\ro{\prd{i}{j}}$ that produces $\pair{i}{j}$.
Similarly, if $\pair{i}{j} \notin \wbest$ there is a unique rotation $\roelt{i}{j}$
 that eliminates $\pair{i}{j}$.
Note that it is always the case that $M$ (strictly) dominates $M / \ro{}$.


There exists a partial order for rotations. 
A rotation  $\rho'$ is said to precede another rotation $\rho$ (denoted by $\rho'  \prec\prec  \rho $),
if $\rho'$ is eliminated in every sequence of eliminations that starts at $\mbest$ and ends at a stable matching in which $\rho$ is exposed~\cite{Gusfield:1989:SMP:68392}.
Note that this relation is transitive, that is, $\rho''  \prec\prec \rho' \wedge \rho'  \prec\prec  \rho  \implies \rho'' \prec\prec  \rho $. 
Two rotations are said to be \textit{incomparable} if one does not precede the other.  
The structure that represents all rotations and their partial order is a directed graph called \textit{rotation poset} denoted by $\Pi = (\roset, E)$. Each rotation corresponds to a vertex in $\roset$ and there exists an edge from $\rho'$ to $\rho$ if $\rho'$ precedes $\rho$. 
The number of rotations is bounded by $n(n-1)/2$ and the number of arcs is $O(n^2)$~\cite{Gusfield:1989:SMP:68392}. 
It should be noted that the construction of $\Pi$ can be done in $O(n^2)$.

Predecessors of a rotation $\ro{}$ in a rotation poset are denoted by $\preNT{\ro{}}$ and successors are denoted by $\sucNT{\ro{}}$.
Later, we shall need transitivity to complete these lists. 
Therefore, we denote by $\preT{\ro{}}$ (respectively $\sucT{\ro{}}$)
the predecessors (respectively successors) of a rotation $\ro{}$ including transitivity.

In Figure~\ref{figure:lattice}, we give the lattice of all stable matchings of the instance given in Table~\ref{table:sm}. There exists two vectors for each stable matching. The first vector represents the set of men and the second vector represents the partner of each man in the matching. 		
Each edge of the form $M$ to $M'$ on the lattice is labeled with the rotation $\rho$ such that $M'$ is obtained after exposing $\rho$ on $M$.
All the rotations are given in Figure~\ref{figure:closedSubset}. The latter represents the rotation poset of this instance.

\begin{figure}[ht]
    \centering
 	\includegraphics[width=.45\textwidth]{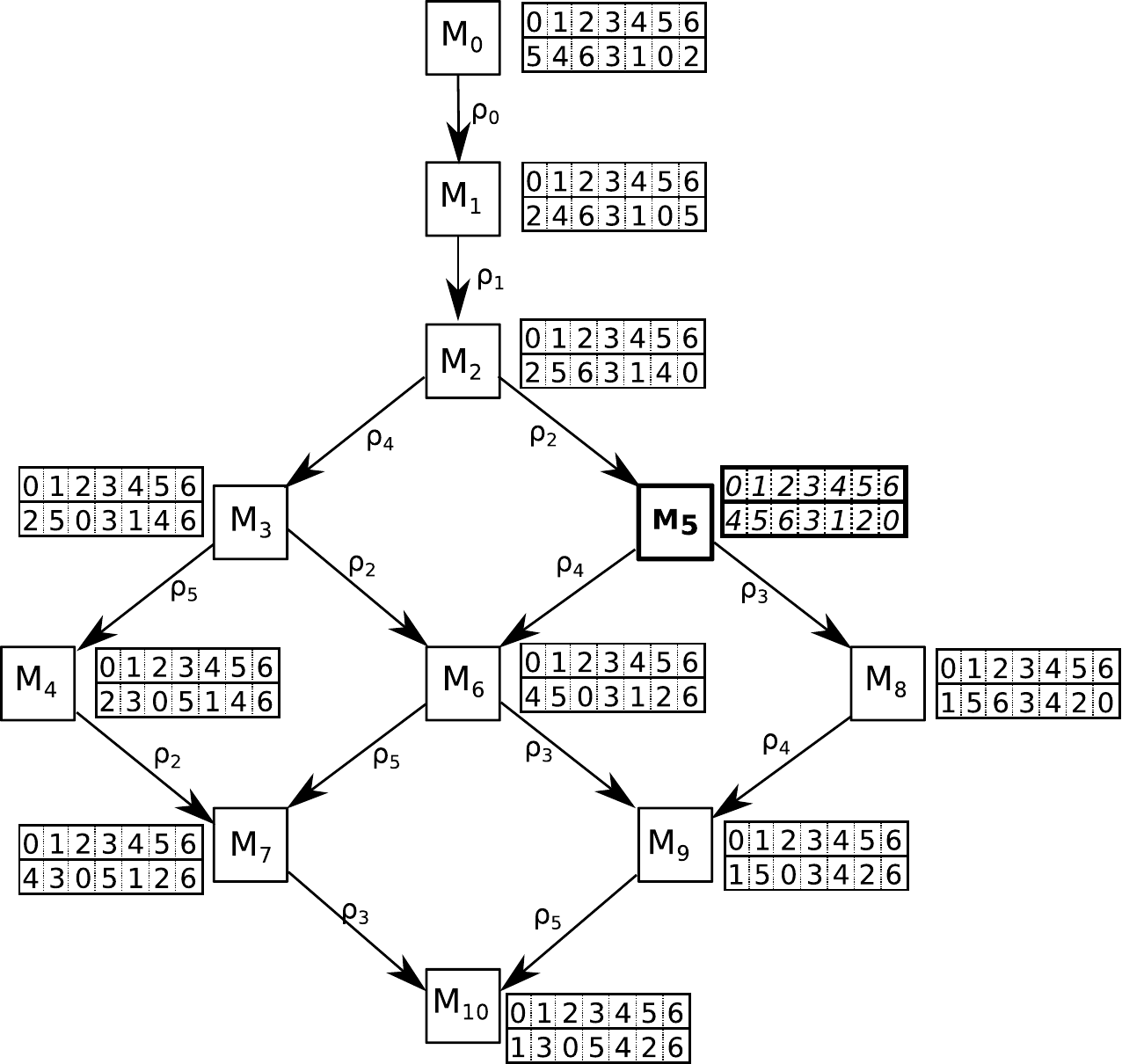}
    \caption{The lattice of all stable matchings corresponding to the instance given in Table~\ref{table:sm}.}
    \label{figure:lattice}
\end{figure}

\begin{figure}[ht]
    \centering
 	\includegraphics[width=.35\textwidth]{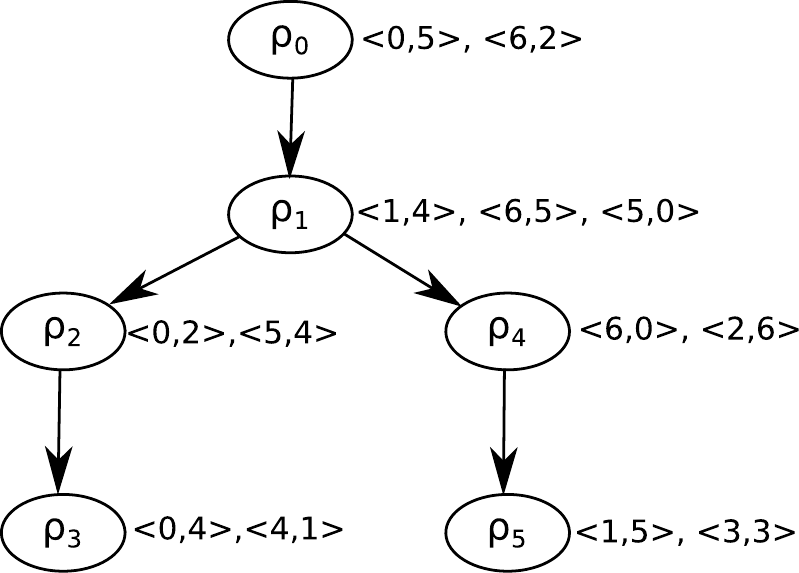}
    \caption{Rotation poset of the instance given in Table~\ref{table:sm}.}
    \label{figure:closedSubset}
\end{figure}

A \textit{closed subset} $S$ is a set of rotations such that for any rotation $\rho$ in $S$, if there exists a rotation $\rho' \in \roset $ that precedes $\rho$ then $\rho'$ is also in $S$. By Theorem~\ref{GFtheo257}, every closed subset in the rotation poset corresponds to a stable matching.
We denote by $X(S)$ the set of men that are included in at least one of the rotations in $S$.

Below is a theorem and a corollary from ~\cite{Gusfield:1989:SMP:68392} mainly used in some proofs in the next section. 


 \begin{theorem}[Theorem 2.5.7] 
 i) There is a one-one correspondence between the closed subsets of $\Pi$ and the stable matchings of $\mathscr{M}$.\\
ii) S is the closed subset of rotations of $\Pi$ corresponding to a stable matching $M$ if and only if $S$ is the (unique) set of rotations on every $\mbest$-chain 
in $\mathscr{M}$ ending at $M$. Further, $M$ can be generated from $\mbest$ by eliminating the rotations in their order along any of these paths, and these are the only ways to generate $M$ by rotation eliminations starting from $\mbest$.\\
 iii) If $S$ and $S'$ are the unique sets of rotations corresponding to distinct stable matchings $M$ and $M'$, then $M$ dominates $M'$ if and only if $S \subset S'$.
 \label{GFtheo257}
 \end{theorem}

 \begin{corollary}
 [Corollary 3.2.1] 
 Every man-woman pair $\noLBLpair{m}{w}$ is in at most one rotation. Hence there are at most $n(n-1)/2$ rotations in an instance of the Stable Marriage problem of size n.
 \label{GFcoro321}
 \end{corollary}

We introduce here a notion that is important to measure robustness. 
Let $\sm{i}$ and $\sm{j}$ be two stable matchings. 
The distance
$d(\sm{i}, \sm{j})$ is the number of men that have different partners in $\sm{i}$ and $\sm{j}$.
A matching $\sm{i}$ is closer to $\sm{j}$ than $\sm{k}$ if
$d(\sm{i}, \sm{j}) < d(\sm{i}, \sm{k}) $.
We give in Definition\ref{def:supersm} a formal definition of $(a,b)$-supermatch.

\begin{definition}{$(a,b)$-supermatch}
\label{def:supersm}
A stable matching $M$ is said to be $(a,b)$-supermatch if 
for any set $\Psi \subset M$ of $a$ stable pairs that are not fixed,
there exists a stable matching $M'$ such that 
$M' \cap \Psi = \emptyset$ and 
$d(M, M') -a  \leq b$.	
\end{definition}

In this paper, we focus on the $(1,b)$-supermatch case. 
Notice that for any stable matching $M$, there exists a value $b$ such that $M$ is a $(1,b)$-supermatch. 
In this case, we say that $b$ is the \textit{robustness value} of $M$.
The \textit{most robust stable matching} is a $(1,b)$-supermatch with the minimum robustness value.

\section{Checking $(1,b)$-supermatch in Polynomial Time}
\label{sec:verification}

A preliminary version of this section appeared in~\cite{DBLP:conf/aaai/Genc0OS17}.

In this section, we first recall the basics, and then show how to find or verify the closest stable matching of a given stable matching with an unwanted couple.
Throughout this section, we suppose that $M$ is a given stable matching, $S$ its corresponding closed subset, and $\pair{i}{j} \in M$ is a non-fixed (stable) pair to remove from $M$.	

The closest stable matching to $M$ that does not include $\pair{i}{j}$ is a matching $M^*$ in which either $\pair{i}{j}$ was eliminated or not produced in any sequence of rotation eliminations starting from $\mbest$ leading to $M^*$.
Hence, if $\pair{i}{j} \notin \mbest$ then $\roprd{i}{j} $ exists, 
and there is a set of stable matchings $S_u$, where each of them dominates $M$ and does not include the $\pair{i}{j}$.
Similarly, if $\pair{i}{j} \notin \wbest$ then $\ro{\elt{i}{j}}$ exists, 
and there is a set of stable matchings $S_d$, where each of them is dominated by $M$ and does not include the pair $\pair{i}{j}$.

If $\pair{i}{j} \notin \mbest$, we define a specific set of rotations $\sManUp{i}$ as follows\footnote{The brackets $()$ in Equations \ref{eq:up} and \ref{eq:down} are used to give the priority between the operators.}:
\begin{equation}
\label{eq:up}
\sManUp{i} = S \setminus ( \{\roprd{i}{j}\} \cup ( \sucT{\roprd{i}{j}} \cap S ) ).
\end{equation}

If $\pair{i}{j} \notin \wbest$, we define a specific set of rotations $\sManDown{i}$ as follows:
\begin{equation}
\label{eq:down}
\sManDown{i} = S \cup   \{\roelt{i}{j}\} \cup ( \preT{\roelt{i}{j}} \setminus S ) .
\end{equation}

Observe first that 
$\sManUp{i}$ and 
$\sManDown{i}$ are in fact closed subsets since $S$ is a closed subset. 
Let $\smup{i}$ (respectively $\smdown{i}$) be the stable matching corresponding to $\sManUp{i}$  (respectively $\sManDown{i}$). By construction, we have $\smup{i} \in S_u$ and $\smdown{i} \in S_d$. 
We show later that any stable matching $\sm{k} \notin \{\smup{i}, \smdown{i} \}$
 that does not include the pair  $\pair{i}{j}$ cannot be closer to $M$ than $\smup{i}$
 or $\smdown{i}$.

For illustration, consider the stable matching $\sm{5}$ of the previous example and its closed subset $S_5 = \{\ro{0}, \ro{1}, \ro{2} \}$. 
Table~\ref{table:allRepairs1} shows for each pair $\pair{i}{j}$
the matchings $\sManUp{i}$ and $\sManDown{i}$ if they exist.

\begin{table}[ht]
\centering
\renewcommand{\tabcolsep}{2pt}
\begin{tabular}{|c|c|c|c|c|}
\hline
$\pair{i}{j} $ & $\roprd{i}{j}$ & $\roelt{i}{j}$ & $\sManUp{i}$ & $\sManDown{i}$ \\ 
\hline
$\pair{0}{4} $ & $\ro{2}$ & $\ro{3}$ & $\{ \ro{0}$, $\ro{1} \}$ & $ \{ \ro{0}$, $\ro{1}$, $\ro{2}$, $\ro{3} \}$ \\ \hline
$\pair{1}{5} $ & $\ro{1}$ & $\ro{5}$ & $\{ \ro{0} \}$ & $\{ \ro{0}$, $\ro{1}$, $\ro{2}$, $\ro{4}$, $\ro{5} \}$ \\ \hline
$\pair{2}{6} $ & - & $\ro{4}$ & - & $\{ \ro{0}$, $\ro{1}$, $\ro{2}$, $\ro{4} \}$ \\ \hline
$\pair{3}{3} $ & - & $\ro{5}$ & - & $\{ \ro{0}$, $\ro{1}$, $\ro{2}$, $\ro{4}$, $\ro{5} \}$ \\ \hline
$\pair{4}{1} $ & - & $\ro{3}$ & - & $\{ \ro{0}$,$\ro{1}$, $\ro{2}$, $\ro{3} \}$ \\ \hline
$\pair{5}{2} $ & $\ro{2}$ & - & $\{ \ro{0}$, $\ro{1} \}$ & - \\ \hline
$\pair{6}{0} $ & $\ro{1}$ & $\ro{4}$ & $\{ \ro{0}$ \} & $ \{ \ro{0}$, $\ro{1}$, $\ro{2}$, $\ro{4} \}$ \\ 
\hline

\end{tabular}
\caption{The repair closed subsets $\sManUp{i}$ and $\sManDown{i}$ for $\sm{5}$.}
\label{table:allRepairs1}
\end{table}

Table~\ref{table:allRepairs2} shows the stable matchings corresponding to the closed subsets given in Table~\ref{table:allRepairs1} and the distances between the current stable matching $\sm{5}$ to each one of them.
The distances are denoted as $\dup{i}$ and $\ddown{i}$ in the table, where for each man $\man{i}$, $\dup{i} = d(\sm{5}, \smup{i})$ and $\ddown{i} = d(\sm{5}, \smdown{i})$, respectively. 
If $\smup{i}$ does not exist for a man $\man{i}$, then $\dup{i}$ is denoted by $\infty$ (the same value is used when $\smdown{m}$ does not exist). 
Last, $b_i = min(\dup{i}, \ddown{i}) -1$ represents the repair cost of each man. The reason for extraction of $1$ is because we are considering only $(1,b)$-supermatches (see Definition~\ref{def:supersm}).


\begin{table}[h!]
\centering
\begin{tabular}{|c|c|c|c|c|}
\hline
\multicolumn{1}{|l|}{$\smup{i}$} & \multicolumn{1}{l|}{$\smdown{i}$} & \multicolumn{1}{l|}{$\dup{i}$} & \multicolumn{1}{l|}{$\ddown{i}$} & \multicolumn{1}{l|}{$b_i$} \\ \hline
$\sm{2}$ & $\sm{8}$ & 2 & 2 & 1 \\ \hline
$\sm{1}$ & $\sm{7}$ & 4 & 4 & 3 \\ \hline
- & $\sm{6}$ & $\infty$ & 2 & 1 \\ \hline
- & $\sm{7}$ & $\infty$ & 4 & 3 \\ \hline
- & $\sm{8}$ & $\infty$ & 2 & 1 \\ \hline
$\sm{2}$ & - & 2 & $\infty$ & 1 \\ \hline
$\sm{1}$ & $\sm{6}$ & 4 & 2 & 1 \\ \hline
\end{tabular}
\caption{The repair stable matchings $\smup{i}$ and $\smdown{i}$ for each man in $\sm{5}$ and the distances between the original stable matching and the repair stable matchings.}
\label{table:allRepairs2}
\end{table}

The robustness of a stable matching is equal to the repair cost of the non-fixed man that has the worst repair cost $b = \sum_{i \in \{1...n\}} max(b_i)$. For the follow-up example, $\sm{5}$ is labelled as a $(1,3)$-supermatch. Moreover, Table~\ref{table:allSMRobustness} shows the $(1,b)$-robustness measures for each stable matching of the given sample. The most robust stable matching for the sample is identified as $\sm{6}$ since it has the smallest $b$ value. 

\begin{table}[ht]
\centering
\begin{tabular}{|c|c|c|}
\hline
SM $(\sm{k})$ & Closed Subset $(S_k)$ & Robustness $(b)$ \\ \hline
$\sm{0}$ & $\{  \}$ & 5 \\ \hline
$\sm{1}$ & $\{ \ro{0} \}$ & 4 \\ \hline
$\sm{2}$ & $\{ \ro{0}, \ro{1} \}$ & 3 \\ \hline
$\sm{3}$ & $\{ \ro{0}, \ro{1}, \ro{4} \}$ & 2 \\ \hline
$\sm{4}$ & $\{ \ro{0}, \ro{1}, \ro{4}, \ro{5} \}$ & 3 \\ \hline
$\sm{5}$ & $\{ \ro{0}, \ro{1}, \ro{2} \}$ & 3 \\ \hline
$\sm{6}$ & $\{ \ro{0}, \ro{1},  \ro{2}, \ro{4} \}$ & 1 \\ \hline
$\sm{7}$ & $\{ \ro{0}, \ro{1}, \ro{2}, \ro{4}, \ro{5} \}$ & 3 \\ \hline
$\sm{8}$ & $\{ \ro{0}, \ro{1}, \ro{2}, \ro{3} \}$ & 3 \\ \hline
$\sm{9}$ & $\{ \ro{0}, \ro{1}, \ro{2}, \ro{3}, \ro{4} \}$ & 2 \\ \hline
$\sm{10}$ & $\{ \ro{0}, \ro{1}, \ro{2}, \ro{3}, \ro{4}, \ro{5} \}$ & 3 \\ \hline
\end{tabular}
\caption{The robustness values of all stable matchings for the sample given in Table~\ref{table:sm}.}
\label{table:allSMRobustness}
\end{table}

We give few lemmas in order to show that the closest stable matching to $M$ when breaking the pair $\pair{i}{j}$ is either $\sManUp{i}$  or $\sManDown{i}$. 

 \begin{lemma} Given two incomparable rotations $\rho$ and $\rho'$, $X(\{\rho\}) \cap X(\{\rho'\}) = \emptyset$.
 \label{incomparableRotations}
 \end{lemma}

 \begin{proof}
By definition of incomparability, if two rotations are incomparable, it means that they modify a set of men who do not require modifications from the other first. Therefore the sets of men are distinct.
 \end{proof}

 \begin{lemma} Given three stable matchings $\sm{x}, \sm{y}$ and $\sm{z}$ 
such that $\sm{x} \preceq \sm{y} \preceq \sm{z}$, then $d(\sm{y}, \sm{z}) \leq d(\sm{x}, \sm{z})$ and $d(\sm{x}, \sm{y}) \leq d(\sm{x}, \sm{z})$.
 \label{closenessLemma}
 \end{lemma}

 \begin{proof}
 Using the properties of domination and the closed subsets in Theorem~\ref{GFtheo257}, we can infer $S_{x} \subset S_{y} \subset S_{z}$. 

 Assume to the contrary that $d(M_{y}, M_{z}) > d(M_{x}, M_{z})$.
 This situation occurs only if a set of pairs that are present in $M_{x}$ are eliminated to obtain $M_{y}$ and then re-matched with the same partners they had in $M_{x}$ to get $M_{z}$. However, this contradicts Corollary~\ref{GFcoro321}. For similar reasons, $d(M_{x}, M_{y}) < d(M_{x}, M_{z})$.
\end{proof}

The case where stable matchings have the same distance such as $d(\sm{x}, \sm{y}) = d(\sm{x}, \sm{z})$, if the rotation set in the difference sets $S_y \setminus S_x$, $S_z \setminus S_x$, and $S_z \setminus S_y$ modify the same set of men. We can demonstrate the case where we have equal distance between the three matchings on a Stable Marriage instance of size $8$ given in Manlove's book on page 91~\cite{manlove2013algorithmics}. It can be verified that $\sm{8} \preceq \sm{15} \preceq \sm{21}$, where: \\
  $\sm{8} = \{\noLBLpair{1}{1},\noLBLpair{2}{3},\noLBLpair{3}{4},\noLBLpair{4}{8},
  \noLBLpair{5}{2},\noLBLpair{6}{5}, \noLBLpair{7}{6},\noLBLpair{8}{7}\}$, \\ 
  $\sm{15} = \{\noLBLpair{1}{5},\noLBLpair{2}{4},\noLBLpair{3}{3},\noLBLpair{4}{6},
  \noLBLpair{5}{8},\noLBLpair{6}{7}, \noLBLpair{7}{2},\noLBLpair{8}{1}\}$, \\ 
  $\sm{21} = \{\noLBLpair{1}{7},\noLBLpair{2}{8},\noLBLpair{3}{2},\noLBLpair{4}{1},
  \noLBLpair{5}{6},\noLBLpair{6}{4},\noLBLpair{7}{3},\noLBLpair{8}{5}\}$.

  In this case, due to the fact that $X(S_{15} \setminus S_{8}) = X(S_{21} \setminus S_{15}) = X(S_{21} \setminus S_{8})$, the distances of all three matchings will be equal to each other $d(M_{8}, M_{15}) = d(M_{15}, M_{21}) = d(M_{8}, M_{21}) = 8$.

\begin{lemma} 
\label{lemma:dominating}
 If there exists a stable matching $\sm{x}$ that does not contain $\pair{i}{j}$, dominates $M$ and different from $\smup{i}$, then $\sm{x}$ dominates $\smup{i}$.
\end{lemma}

\begin{proof}
 $\smup{i} \preceq M$ by definition.
 Suppose by contradiction that there exists an $\sm{x}$ such that $\pair{i}{j} \not\in \sm{x}$ and
  $\smup{i} \preceq \sm{x} \preceq M$. 
 It implies that $\sManUp{i} \subset S_x \subset S$. 
 In this case, $(S_x\setminus\sManUp{i})\subset\Big\{  \{\ro{\prd{i}{j}}\} \cup \{ \sucT{\ro{\prd{i}{j}}} \cap S \}  \Big\}$. 
 However, this set contains $\ro{\prd{i}{j}}$ and the rotations preceded by $\ro{\prd{i}{j}}$. 
 Adding any rotation from this set to $S_x$ results in a contradiction by either adding
  $\pair{i}{j}$ to the matching, thereby not breaking that couple,
  or because the resulting set is not a closed subset.  
\end{proof} 

 \begin{lemma} 
 \label{lemma:dominated}
 If there exists an $\sm{x}$ that does not contain $\pair{i}{j}$ dominated by $M$ but different from $\smdown{i}$, then $\smdown{i}$ dominates $\sm{x}$.
 \end{lemma}

 \begin{proof}
 Similar to the proof above, suppose that there exists an $\sm{x}$ such that $\pair{i}{j} \not\in \sm{x}$
  and $M \preceq \sm{x} \preceq \smdown{i}$. 
 We have $S \subset S_x \subset \sManDown{i}$. 
 It implies $(S_x \setminus S) \subset \Big\{  \{\ro{\elt{i}{j}}\} \cup \{ \preT{\ro{\elt{i}{j}}} \setminus S \} \Big\}$.
 This set contains the rotation $\ro{\elt{i}{j}}$ that eliminates the pair and the rotations preceding $\ro{\elt{i}{j}}$. 
In order to add $\ro{\elt{i}{j}}$ all other rotations must be added to form a closed subset. 
 If all rotations are added, $S = \sManDown{i}$ which results in a contradiction.
 \end{proof}

\begin{lemma} 
\label{lemma:incomparable}
 For any stable matching $\sm{k}$ incomparable with $M$ such that $\sm{k}$ does not contain the pair $\pair{i}{j}$, $\smup{i}$ is closer to $M$ than $\sm{k}$.
\end{lemma}

\begin{proof} Let $S_k$ be the closed subset corresponding to $\sm{k}$, and $S$ be that corresponding to $M$. 

First, we consider the case in which $S_k \cap S = \emptyset$. 
If the closed subsets have no rotations in common the rotations in these sets are incomparable.
Using Lemma~\ref{incomparableRotations}, $X(S_k) \cap X(S) = \emptyset$. 
Therefore, $d(\sm{k}, M) = |X(S_k)| + |X(S)|$, whereas $d(\smup{i}, M) \leq |X(S)|$.
 
Second, we consider the case in which $S_k \cap S \neq \emptyset$. 
Let $\sm{c}$ be the closest dominating stable matching of both $\sm{k}$ and $\smup{i}$, along with $S_c$ as its corresponding closed subset.
Using Lemma~\ref{closenessLemma} we know that $d(\smup{i}, \sm{}) \leq d(\sm{c}, \sm{})$,
 where $d(\sm{c}, \sm{}) = |X(S \setminus S_c)|$.

Using Lemma~\ref{incomparableRotations} we know that $X(S_k \setminus S_c) \cap
  X(S \setminus S_c) = \emptyset $. 
Therefore, $d(\sm{k}, M) = |X(S_k \setminus S_c)| + |X(S \setminus S_c)|$. 
By substituting the formula above, $d(\sm{k}, M) \geq |X(S_k \setminus S_c)| + d(\smup{i}, M)$. 
Using the fact that $|X(S_k \setminus S_c)| > 0$ from the definition of $\sm{k}$, we can conclude that $d(\sm{k}, M) > d(\smup{i}, M)$.
\end{proof}

\smallskip\noindent
The following theorem is a direct consequence of Lemmas~\ref{lemma:dominating},~\ref{lemma:dominated},~\ref{lemma:incomparable}. 

\begin{theorem} The closest stable matching of a stable matching $M$ given the unwanted pair $\pair{i}{j}$ is either $\smup{i}$ or $\smdown{i}$.
\end{theorem}

We show that checking if a stable matching is a $(1,b)$-supermatch can be performed in
 $O(n \times |\roset|)$ time after a $O(n^2 + |\roset|^2)$ preprocessing step. 
First, the pre-processing step consists of building: the poset graph (in $O(n^2)$ time); the lists
 $\preT{\ro{}}$, $\sucT{\ro{}}$ for each rotation $\ro{}$ (in $O(|\roset|^2)$ time); and $\roelt{i}{j}$
 and $\roprd{i}{j}$ for each pair $\pair{i}{j}$ whenever applicable (in $O(n^2)$ time).
Next, we compute $\sManUp{i}$ and $\sManDown{i}$ for each man. 
Note that $\sManUp{i}$ and $\sManDown{i}$ can be constructed in
 $O(|\roset|)$ time (by definition of $\sManUp{i}$ and $\sManDown{i}$) for each man $\man{i}$.
Note that $d(\smup{i}, M)$ is equal to the number of men 
participating in the rotations that are eliminated from $S$ to obtain $\sManUp{i}$. 
Similarly, $d(\smdown{i}, M)$ is equal to the number of men 
participating in the rotations that are added $S$ to obtain $\sManDown{i}$.
Last, if $b < d(\smup{i}, M) - 1 $ and $b < d(\smdown{i}, M) - 1$, we know that it is impossible
 to repair $M$ when $\man{i}$ needs to change his partner with at most $b$ other changes.
Otherwise, $M$ is a $(1,b)$-supermatch. 

The constraint programming (CP) formulation and the metaheuristic approaches that we propose in the following sections to find the most robust stable matching are essentially based on this procedure. 



\section{A Constraint Programming Formulation}
\label{sec:CP}

Constraint programming (CP) is a powerful technique for solving combinatorial search problems by expressing the relations between decision variables as constraints~\cite{Rossi:2006:HCP:1207782}.
We give in this section a CP formulation for finding the most robust stable matching, that is, the $(1,b)$-supermatch with the minimum b.
The idea is to formulate the stable matching problem using rotations, 
then extend that formulation in order to compute the two values $d(\smup{i}, M) $ and $d(\smdown{i}, M)$ (where M is the solution) for each man $\man{i}$ so that $b$ is always greater or equal to one of them. 
Using rotations to model stable matching problems has been used in~\cite{Gusfield:1989:SMP:68392} (page 194), and in~\cite{92-feder,07-fleiner,17-m2m}.




 Let $SP$ be the set of all stable pairs that are not fixed and $NM$ be the set of all men that are at least in one of the pairs in $SP$. 
Let the set of rotations be $\{ \ro{1},  \ro{2}, \ldots,  \ro{|\roset|} \}$ and 
for each $\pair{i}{j} \notin M_0$, $\roprd{i}{j}$ is the unique rotation that produces
 $\pair{i}{j}$.
Moreover, for each $\pair{i}{j} \notin M_z$, $\roelt{i}{j}$ denotes the unique rotation that eliminates
 the pair $\pair{i}{j}$. 
Notice that $\prd{i}{j}, \elt{i}{j} \in \{1, \ldots, |\roset| \}$.
We use $R_i$ to denote the set of rotations in which man $\man{i}$ is involved.


In our model, we assume that a $O(n^2)$ pre-processing step is performed to compute $\mbest$, $\wbest$, 
$SP$, $NM$, the poset graph, $\roelt{i}{j}$, and $\roprd{i}{j}$ for every pair $\pair{i}{j}$ whenever applicable.
We also need to compute the lists  $\preT{\ro{}}$, $\sucT{\ro{}}$ for each rotation $\ro{}$ (in $O(|\roset|^2)$ time). 

We start by giving the different sets of variables, then we show the constraints. 
Let $\woman{i}^0$ denote man $\man{i}$'s partner in $\mbest$, and $\woman{i}^Z$ denote his partner in $\wbest$.
 There exists a boolean variable $\x{i}{j}$ for every pair $\pair{i}{j}$ to indicate if $\pair{i}{j}$ is part of the solution $M$. 
Therefore, $\x{i}{j}$ is a boolean variable related to a pair in $SP$.
There exists also a boolean variable $\s{v}$ for each rotation $\ro{v}$, where
 $v \in \{ 1, \ldots, \vert \roset \vert \}$ to indicate if $\ro{v}$ is in the closed subset of the solution. 
Moreover, we use an integer variable $b$ with $[1,\vert NM \vert - 1]$ as an initial domain to represent the objective. That is, the solution is a $(1,b)$-supermatch with the minimum $b$. 

 For each man $\man{i} \in NM$, we define the following variables:  

\begin{itemize}
   \item  $\alpha_i$ (respectively $\beta_i$): a boolean variable to indicate if $\pair{i}{\woman{i}^0}$ (respectively  $\pair{i}{\woman{i}^Z}$) is a part of the solution. 

\item  $\sUp{v}{i}$ (respectively $\sDown{v}{i}$): a boolean variable for every $\ro{v} \in \roset$
 to know if $\ro{v}$ is in the difference set $S \setminus \sManUp{i}$ (respectively
 $\sManDown{i} \setminus S$) where S is the closed subset of the solution, $\sManUp{i}$
 (respectively $\sManDown{i}$) corresponds to the closed subset of the nearest stable matchings
 dominating (respectively dominated by) S when then pair defined by man $\man{i}$ and his partner is to be eliminated. 


\item  $\y{l}{i}$ (respectively $\z{l}{i}$): a boolean variable for each man $\man{l} \in NM$ to know if
 the partner of $\man{l}$ needs to be changed to find a repair when man $\man{i}$ breaks up with his
 partner to obtain $\sManUp{i}$ (respectively $\sManDown{i}$) from the solution. 


\end{itemize}
  
A pair is a part of a solution if and only if it is produced by a rotation (if it is not a part of $\mbest$) and not eliminated by another (if it is not a part of $\wbest$). Moreover, a set of rotations corresponds to a closed subset (or a solution) if and only if all the parents of each rotation in the set are also in the set. Hence, the following constraints are required to keep track of the pairs that are in the solutions and also to make sure the solution is a stable matching.

More specifically, let $S$ denote the closed subset of the solution. 
Without loss of generality, the first constraints ensure that if a pair $\pair{i}{j}$ is in the solution, the rotation that produces $\pair{i}{j}$ must be in $S$ and the rotation that eliminates it should not be in $S$.
The second constraints enforce the set of rotations assigned to true to form a closed subset.
This model was used in the many-to-many setting in~\cite{17-m2m}.

\begin{enumerate}
\item \label{cons:stability}  $\forall \pair{i}{j} \in SP$:
     \begin{enumerate}
     \item \label{cons:stability1} If $\pair{i}{j} \in \mbest$, $\x{i}{j} \leftrightarrow \neg \s{\elt{i}{j}}$
     \item \label{cons:stability2}  Else if $\pair{i}{j} \in \wbest$, $\x{i}{j} \leftrightarrow \s{\prd{i}{j}}$
     \item \label{cons:stability3}  Otherwise $\x{i}{j} \leftrightarrow (\s{\prd{i}{j}} \wedge \neg \s{\elt{i}{j}})$    
     \end{enumerate}

\item \label{cons:closed}$\forall \ro{v}, \forall \ro{v'} \in \preNT{\ro{v}}$: $\s{v} \rightarrow \s{v'}$. 



\item The following set of constraints are required to build the difference set $S \setminus \sManUp{i}$:
\begin{enumerate}

\item $\forall \man{i} \in NM, \alpha_i \leftrightarrow \x{i}{\woman{i}^0}$: 
This constraint handles the special case when man $\man{i}$ is matched to his partner in $\mbest$. 

\item $\forall \man{i} \in NM$:
    \begin{enumerate}
    \item \label{cons:produce} $\forall \pair{i}{j} \in SP$, where $j \not= \woman{i}^0$:     
         $ \x{i}{j} \leftrightarrow \sUp{\prd{i}{j}}{i} $. 

    This constraint is used to know the rotation that produces the partner of $\man{i}$ in the solution.
         
    \item $\forall \ro{v} \in \roset:$ 
    \begin{itemize}
	  \item $ \alpha_i \rightarrow \neg \sUp{v}{i}$ $\forall \ro{v} \in \roset$: to make sure if this pair is not produced by a rotation, then there does not exist an $\sManUp{i}$ set.

 	 \item  $\sUp{v}{i} \wedge \s{v'} \rightarrow \sUp{v'}{i}$ ,$\forall  \ro{v'} \in \sucNT{\ro{v}}$: to build the difference set.

 	\item There are two cases to distinguish:

    \begin{itemize}
    
\item  If $\ro{v}$ produces the pair $\pair{i}{j}$:
        $\sUp{v}{i} \rightarrow \s{v} \wedge \big( \x{i}{j}  \vee \bigvee_{ \ro{v*} \in \preNT{\ro{v}}} (\sUp{v*}{i} \big)$.
	
    \item Else :
  \label{cons:notproduces}  $\sUp{v}{i} \rightarrow \s{v} \wedge \big( \bigvee _{\ro{v*} \in \preNT{\ro{v}}} (\sUp{v*}{i} \big)$.

	\end{itemize}
    \end{itemize}
        






\end{enumerate}
\end{enumerate}

\item Following the same intuition above, we use the following constraints to represent the difference set $\sManDown{i} \setminus S$.

\begin{enumerate}
\item $\forall \man{i} \in NM, \beta_i \leftrightarrow \x{i}{\woman{i}^Z}$.

\item $\forall \man{i} \in NM$:
    \begin{enumerate}

     \item $\forall \pair{i}{j} \in SP$, where $j \not= \woman{i}^Z$: $ \x{i}{j} \leftrightarrow \sDown{\elt{i}{j}}{i} $ 
     
    \item $ \forall \ro{v} \in \roset$:
    
    \begin{itemize}                 
  	 	\item $ \beta_i \rightarrow \neg \sDown{v}{i}$: to handle the particular case of having a partner in $M_z$. 
   		\item $ \sDown{v}{i} \wedge \neg \s{v'} \rightarrow \sDown{v'}{i}$ $\forall \ro{v'} \in \preNT{\ro{v}}$. 
        
        \item There are two cases to distinguish:
    
    \begin{itemize}
        
        \item If $\ro{v}$ eliminates the pair $\pair{i}{j}$:
        $\sDown{v}{i} \!\!\!\rightarrow\! \neg \s{v} \wedge\! \big( \x{i}{j} \!\vee \bigvee_{\ro{v*} \in \sucNT{\ro{v}}} \sDown{v*}{i} \big)$.
        
   		\item  Else:
        $\sDown{v}{i} \!\!\!\rightarrow\! \neg \s{v} \wedge\! \big( \bigvee_{\ro{v*} \in \sucNT{\ro{v}}} \sDown{v*}{i} \big)$.
\end{itemize}
    
        
         
     
\end{itemize}
    \end{enumerate}
    
\end{enumerate}


\item Last, we define the following constraints to count exactly how many men must be modified to repair the solution by using $\sManUp{i}$ and $\sManDown{i}$: 

\begin{enumerate}
\item The following two constraints are required to keep track if man $\man{l}$ has to change his partner to obtain $\sManUp{i}$ (or $\sManDown{i}$) upon the break-up of man $\man{i}$ with his current partner. $\forall \man{l} \neq \man{i} \in NM$: 
   \begin{enumerate}
   \item
      $\y{l}{i} \leftrightarrow \bigvee_{\ro{v} \in R_l} \sUp{v}{i}$

     \item   $\z{l}{i} \leftrightarrow \bigvee_{\ro{v} \in R_l} \sDown{v}{i}$  
\end{enumerate}

\item The next two constraints are required to count the number of men that needs to change their partners to obtain the closest stable matchings. $\forall \man{i} \in NM$: 
\begin{enumerate}
   \item      $\alpha_i \rightarrow \dup{i} = n$ and
      $\neg \alpha_i \rightarrow \dup{i} = \sum \y{l}{i}$

   \item      $\beta_i \rightarrow \ddown{i} = n$ and 
      $\neg \beta_i \rightarrow \ddown{i} = \sum \z{l}{i} $\\
\end{enumerate}
   
\item To constrain b:

$\forall \man{i} \in NM$,    
 $\big(\neg \alpha_i \rightarrow (b \geq \dup{i}) \big) \vee  \big(\neg \beta_i \rightarrow (b \geq \ddown{i}) \big)$

\item The objective is to minimise $b$. We can easily improve the initial upper bound by taking the minimum value between the maximum number of men required to repair any (non-fixed) pair from $\mbest$ and from $\wbest$.

\end{enumerate}

\end{enumerate}

\section{Metaheuristics}
\label{sec:Metaheuristics}
For finding the $(1,b)$-supermatch with minimal $b$, we provide solutions using two different mono-objective metaheuristics: genetic algorithm (GA) and iterated local search (LS)~\cite{Boussaid201382}. 
The intuition behind using those approaches is that they are mostly based on finding neighbour solutions and are easy to adapt to our problem. 
In this section, we will describe how to apply those two methods to find $(1,b)$-supermatches and subsequently the most robust stable matching.

The two metaheuristic models described in the following sections are also based on the assumption that a $O(n^2)$ pre-processing step is performed to compute $\mbest$, $\wbest$, the poset graph, the lists  $\preT{\ro{}}$, $\sucT{\ro{}}$ for each rotation $\ro{}$ (in $O(|\roset|^2)$ time).

\subsection{Genetic Algorithm}
\label{meta:geneticAlg}

Genetic algorithms are being used extensively in optimization problems as an alternative to traditional heuristics. They are compelling robust search and optimization tools, which work on the natural concept of evolution, based on natural genetics and natural selection.
Holland introduced the GAs and he has also shown how to apply the process to various computationally difficult problems~\cite{holland1975,holland92adaptation,Back:1996:EAT:229867}.

Our approach is based on the traditional GA, 
where evolutions are realised after the initialisation of a random set of solutions (population). 
The process finishes either when an acceptable solution is found 
or the search loop reaches a predetermined limit, 
i.e. maximum iteration count or time limit.



The genetic algorithm that we propose to find the most robust stable matching is detailed below with a follow-up example for one iteration.
 
\subsubsection{Initialization}
The purpose of initialization step is to generate a number of random stable matchings for constructing the initial population set, denoted by $\boldsymbol{P}$.
Recall that each closed subset in the rotation poset corresponds to a stable matching. 
The random stable matching generation is performed by selecting a random rotation $\ro{}$ from the rotation poset and adding all of its predecessors to the rotation set to construct a closed subset $S$, therefore a stable matching $M$ as detailed in Algorithm~\ref{alg:randSMcreation}.
In Line~\ref{lab:createsm}, the method $\Call{CreateSM}{S}$ is responsible from converting the given closed subset $S$ to its corresponding matching $M$ by exposing all the rotations in $S$ in order on $\mbest$. 

\begin{algorithm}
\caption{Random Stable Matching Creation}\label{alg:randSMcreation}
\begin{algorithmic}[1]
\Procedure{CreateRandomSM}{\null}

\State $\ro{} \leftarrow$ \Call{selectRandom}{$\roset$}
\State $S \leftarrow \{ \ro{} \}$

\For {$\ro{p} \in \Call{allPredecessors}{\ro{}}$}
 \State $S \leftarrow S \cup \{ \ro{p} \}$
\EndFor

\State $M \leftarrow$ \Call{CreateSM}{S}
\label{lab:createsm}

\Return $M$

\EndProcedure
\end{algorithmic}
\end{algorithm}

Then, $n$ randomly created stable matchings are added to the population set $\boldsymbol{P}$ as described in Algorithm~\ref{alg:initialization}.

\begin{algorithm}
\caption{Initialization of the Population}\label{alg:initialization}
\begin{algorithmic}[1]
\Procedure{Initialize}{n}

\State $\boldsymbol{P} \leftarrow \emptyset$
\State $i \leftarrow 0$

\While {$i < n$}
 \State $\sm{i} \leftarrow$ \Call{CreateRandomSM}{ }
 \State $\boldsymbol{P} \leftarrow \boldsymbol{P} \cup \{ \sm{i} \}$
 \State $i \Leftarrow i + 1$
\EndWhile

\Return $\boldsymbol{P}$

\EndProcedure
\end{algorithmic}
\end{algorithm}

Let us illustrate the procedure on a sample for the Stable Marriage instance given in Table~\ref{table:sm}. 
Assume that the population size $n$ is 3. In order to generate the 3 initial stable matchings, 3 random
rotations are selected from the rotation poset, namely: $\ro{0}, \ro{1}, \ro{5}$. 
Since $\ro{0}$ has no predecessors, the stable matching that $\ro{0}$ corresponds is obtained by exposing $\ro{0}$ on $\mbest$, resulting in $\sm{1}$. Similarly, the closed subset obtained by adding all predecessors of $\ro{1}$ corresponds to $\sm{2}$ and the closed subset for $\ro{5}$ corresponds to $\sm{4}$.
Thus, the current population is: $\boldsymbol{P} = \langle \sm{1}, \sm{2}, \sm{4} \rangle$.

\subsubsection{Evaluation}
For each stable matching $\sm{i}$, we denote by $b_i$ its robustness value for (1,b)-supermatch.
At the evaluation step, we compute the value $b_i$ of each stable matching $\sm{i} \in \boldsymbol{P}$.
Then, a fitness value $v_i$ is assigned to each $\sm{i}$ in the population.
This is to indicate that a stable matching with a lower b value is more fit.
Then, the values $v_i$ are normalised in the interval $[0,1]$.

Let $max_b$ denote the maximum b value in the population. 
The first thing required is to get the complements of the numbers with $max_b$ to indicate that a stable matching with a smaller b is a more fit solution. 
In the meantime, a small constant $c_o$ is added to each $v_i$ to make sure even the least fit stable matching $M_i$, where $b_i = max_b$, is still an eligible solution (Equation~\ref{eq:complement}). 
Then, the normalization is applied as shown in Equation~\ref{eq:normalization}.
For normalization, an integer variable $sum_v$ is used to keep track of the sum of the total fitness values in the population.

\begin{equation}
v_i = max_b + c_0 - b_i
\label{eq:complement}
\end{equation}

\begin{equation}
v_i = \dfrac{v_i}{\sum_{\sm{i} \in \boldsymbol{P}} v_i}
\label{eq:normalization}
\end{equation}

The steps are detailed in Algorithm~\ref{alg:evaluation}.

\begin{algorithm}
\caption{Evaluation of Fitness Values of each Stable Matching}\label{alg:evaluation}
\begin{algorithmic}[1]
\Procedure{Evaluation}{$ $}

\State $max_b \leftarrow 0$ 
\For{ $\sm{i} \in \boldsymbol{P} $ }
 \If {$b_{i} > max_b$}
  \State $max_b \leftarrow b_{i}$
 \EndIf
\EndFor

\State $sum_v \leftarrow 0$
\label{lab:compBeginning}
\For{ $\sm{i} \in \boldsymbol{P} $ }
 \State $v_{i} \leftarrow (max_b + c_0 - b_{i})$
 \State $sum_v \leftarrow (sum_v + v_{i})$
\EndFor
\label{lab:compEnd}

\For{ $\sm{i} \in \boldsymbol{P} $ }
 \State $v_{i} \leftarrow (v_{i} / sum_v)$
\EndFor

\EndProcedure
\end{algorithmic}
\end{algorithm}

The b values for the stable matchings given in our example, where $\boldsymbol{P} = \langle \sm{1}, \sm{2}, \sm{4} \rangle$ are calculated as $b_1 = 4$, $b_2 = 3$, $b_4 = 3$ by the polynomial procedure defined in Section~\ref{sec:verification}.
Given these values, $b_{max} = 4$. 
Let us fix $c_0$ to $0.5$.
Then, the fitness values are calculated as $v_1 = 4 + 0.5 - 4 = 0.5$, $v_2 = 4 + 0.5 - 3 = 1.5$, $v_4 = 4 + 0.5 - 3 = 1.5$ and the sum of the fitness values is $sum_v = v_1 + v_2 + v_4 = 3.5$.
By normalizing the fitness values, we obtain the following values: 
$v_1 = 0.143, v_2 = 0.429, v_4 = 0.429$ indicating that $\sm{2}$ and $\sm{4}$ are the fittest stable matchings in $\boldsymbol{P}$.

\subsubsection{Evolution}
The evolution step consists in selecting stable matchings from the population $\boldsymbol{P}$ using a selection method, then applying crossover and mutation on the selected matchings.

We use the roulette wheel selection in our procedure. Because it is a method favouring more fit solutions to be selected and therefore useful for the convergence of the population to a fit state~\cite{Goldberg:1989:GAS:534133}.
Algorithm~\ref{alg:rwSelection} gives details of the roulette wheel selection procedure we used. 
First, a random double number $r \in [0,1]$ is generated. 
Then, the fitness values of the stable matchings in the population are visited.
At each step, a sum of the fitness values seen so far are recorded in a variable named $sum_c$ .
This procedure continues until a stable matching $\sm{i}$ is found, where the recording $sum_c$ up to $\sm{i}$ satisfies the criterion at line~\ref{lab:criterionRWS} in Algorithm~\ref{alg:rwSelection}.

\begin{algorithm}
\caption{Roulette Wheel Selection}\label{alg:rwSelection}
\begin{algorithmic}[1]
\Procedure{Selection}{$ $}

\State $r \leftarrow $ \Call{RandomDouble}{$0, 1$}
\State $sum_c \leftarrow 0$ 

\For{$\sm{i} \in \boldsymbol{P}$}

 \If{$(sum_c \leq r) \wedge (r < sum_c + v_{i})$}
 \label{lab:criterionRWS}
   \Return $\sm{i}$ 
 \EndIf
 \State $sum_c \leftarrow sum_c + v_{i}$
\EndFor

\EndProcedure
\end{algorithmic}
\end{algorithm}

The evolution step is then continued by applying crossover and mutation on selected stable matchings.
First, two stable matchings are selected, namely $\sm{1}$ and $\sm{2}$. 
If these matchings are different from each other, crossover is applied (between Lines~\ref{lab:crossStart} and~\ref{lab:crossEnd} in Algorithm~\ref{alg:evolution}).

\begin{algorithm}
\caption{Evolution Phase}\label{alg:evolution}
\begin{algorithmic}[1]
\Procedure{Evolution}{$ $}

\State $\sm{1} \leftarrow$ \Call{Selection}{$ $}
\label{lab:crossStart}
\State $\sm{2} \leftarrow$ \Call{Selection}{$ $}

\If {$ \sm{1} \neq \sm{2}$}
 \State $(\sm{c1}, \sm{c2}) \leftarrow$ \Call{Crossover}{$\sm{1}, \sm{2}$}
 \label{lab:crossEnd}
 \State \Call{Refine}{$\sm{c1}, \sm{c2}$}
 \State \Call{Evaluation}{$ $} 
 \label{lab:evlAfterRef}
\EndIf

\State $\sm{fit} \leftarrow$ \Call{GetFittest}{$\boldsymbol{P}$}
\label{lab:mutStart}
\State $\sm{m} \leftarrow$ \Call{Selection}{$ $}

\If{$\sm{m} \neq \sm{fit}$}
 \State \Call{Mutation}{$\sm{m}$}
 \label{lab:mutEnd}
\EndIf

\EndProcedure
\end{algorithmic}
\end{algorithm}

After applying the crossover on $\sm{1}, \sm{2}$ and obtaining the resulting stable matchings $\sm{c1}, \sm{c2}$, we refine the population.
Refinement is applied for adding adding more fit solutions to the population and removing the least fit solutions.
The overall procedure is detailed in Algorithm~\ref{alg:refine}.
In the refinement procedure, $\sm{c1}, \sm{c2}$ are immediately added to the population $\boldsymbol{P}$.
Then, the worst two solutions, in other words the two least fit stable matchings $\sm{l1}$ and $\sm{l2}$ are found and removed from the population.
Note that after the refinement procedure in line~\ref{lab:evlAfterRef} in Algorithm~\ref{alg:evolution}, the evaluation must be repeated to properly evaluate the current fitness values in the population.

\begin{algorithm}
\caption{Refinement Procedure}\label{alg:refine}
\begin{algorithmic}[1]
\Procedure{Refine}{$\sm{c1}, \sm{c2}$}

\State $\boldsymbol{P} \leftarrow \boldsymbol{P} \cup \{ \sm{c1} \}$
\State $\boldsymbol{P} \leftarrow \boldsymbol{P} \cup \{ \sm{c2} \}$

\State $(\sm{l1}, \sm{l2}) \leftarrow$ \Call{LeastFitTwoSMs}{$\boldsymbol{P}$}

\State $\boldsymbol{P} \leftarrow \boldsymbol{P} \setminus \{ \sm{l1} \}$
\State $\boldsymbol{P} \leftarrow \boldsymbol{P} \setminus \{ \sm{l2} \}$

\EndProcedure
\end{algorithmic}
\end{algorithm}

Subsequently, between the lines~\ref{lab:mutStart} and~\ref{lab:mutEnd} of Algorithm~\ref{alg:evolution}, a stable matching $\sm{m}$ is selected for mutation.
If $\sm{m}$ is different from the fittest stable matching $\sm{fit}$, the mutation is applied.
Mutation step may result in producing either better or worse solutions. 
We do not apply a refinement procedure after the mutations, instead directly apply the mutation on $\sm{m}$.
Next, we go into the details of the crossover and mutation operations.

\paragraph{Crossover.}
The procedure for crossover is detailed in Algorithm~\ref{alg:crossover}. 
Given two stable matchings $\sm{1}, \sm{2}$ and their corresponding closed subsets $S_1$ and $S_2$, 
one random rotation is selected from each subset. 
Let $\ro{1}$ and $\ro{2}$ denote the randomly selected rotations. 
Between the lines~\ref{lab:COifBeg} and~\ref{lab:COifEnd}, the procedure for creating a new closed subset $S_{c2}$ from the closed subset $S_2$ is shown.
The method is, if $S_2$ already contains the rotation $\ro{1}$, then $S_{c2}$ is constructed by removing $\ro{1}$ and all its successor rotations that are in $S_2$ from the set $S_2$. 
We denote this removal process by the method $\Call{Remove}$ in lines~\ref{lab:rem1} and~\ref{lab:rem2} in Algorithm~\ref{alg:crossover}.
However, if $\ro{1} \not\in S_2$, then $\ro{1}$ and all its predecessors that are not included in $S_2$ are added to $S_{c2}$ with the rotations that are already in $S_2$ to form the new closed subset $S_{c2}$. 
Similarly, we denote the addition process by the method $\Call{Add}$ in lines~\ref{lab:add1} and~\ref{lab:add2}.
This new closed subset $S_{c2}$ corresponds to one of the stable matchings produced by crossover, $\sm{c2}$.
Similarly, the same process is repeated for $S_1$ by constructing $S_{c1}$, and another stable matching $\sm{c2}$ is obtained.

\begin{algorithm}
\caption{Crossover Procedure}\label{alg:crossover}
\begin{algorithmic}[1]
\Procedure{Crossover}{$\sm{1}, \sm{2}$}

\State $\ro{1} \leftarrow $ \Call{Random}{$S_1$}
\State $\ro{2} \leftarrow $ \Call{Random}{$S_2$}

\If{$\ro{1} \in S_2$}
\label{lab:COifBeg}
 \State $S_{c2} \leftarrow$ \Call{Remove}{$S_2, \ro{1}$}
 \label{lab:rem1}
\Else
 \State $S_{c2} \leftarrow$ \Call{Add}{$S_2, \ro{1}$}
  \label{lab:add1}
\label{lab:COifEnd}
\EndIf

\If{$\ro{2} \in S_1$}
 \State $S_{c1} \leftarrow$ \Call{Remove}{$S_1, \ro{2}$}
  \label{lab:rem2}
\Else
 \State $S_{c1} \leftarrow$ \Call{Add}{$S_1, \ro{2}$}
 \label{lab:add2}
\EndIf

\State $\sm{c1} \leftarrow$ \Call{CreateSM}{$S_{c1}$}
\State $\sm{c2} \leftarrow$ \Call{CreateSM}{$S_{c2}$}

\Return $(\sm{c1}, \sm{c2})$

\EndProcedure
\end{algorithmic}
\end{algorithm}

For the sake of the example, assume that $\sm{1}$ and $\sm{4}$ are selected for crossover by the roulette wheel selection from the current population $\boldsymbol{P} = \langle \sm{1}, \sm{2}, \sm{4} \rangle$. One random rotation is selected from each set: $\ro{0} \in S_1$ and $\ro{4} \in S_4$. 
Then, the produced stable matchings are $\sm{3}$ since $\ro{4} \not \in S_1$ and $\{ \ro{0} \} \cup \{ \ro{4}, \ro{1} \} = S_3$, and $\sm{0}$ by removing all rotations $\ro{0}, \ro{1}, \ro{4}, \ro{5}$ from $S_4$. 
After adding new stable matchings, the population becomes $\boldsymbol{P} = \langle \sm{1}, \sm{2}, \sm{4}, \sm{0}, \sm{3} \rangle$.
Then the population is refined by removing the least two fit stable matchings, which are in this case $\sm{0} since b_0 = 5$ and $\sm{1} since b_1 = 4$.
Hence, the refined population after crossover is: $\boldsymbol{P} = \langle \sm{2}, \sm{4}, \sm{3} \rangle$.

\paragraph{Mutation.}
At this step, the mutation method is given a stable matching $\sm{}$. Let $S$ denote its corresponding closed subset.
The procedure starts by selecting a random rotation $\ro{}$ from the rotation poset $\roset$.
Then, similar to the crossover procedure above, if $\ro{} \not\in S$, $\ro{}$ and all its required predecessors to form a closed subset are added to $S$.
However, if $\ro{} \in S$, then $\ro{}$ and all its successors in $S$ are removed from $S$. 
Algorithm~\ref{alg:mutation} defines this procedure.

\begin{algorithm}
\caption{Mutation Procedure}\label{alg:mutation}
\begin{algorithmic}[1]
\Procedure{Mutation}{$\sm{}$}

\State $\ro{} \leftarrow $ \Call{Random}{$\roset$}

\If{$\ro{} \in S$}
 \State $S \leftarrow$ \Call{Remove}{$S, \ro{}$}
\Else
 \State $S \leftarrow$ \Call{Add}{$S, \ro{}$}
\EndIf

\EndProcedure
\end{algorithmic}
\end{algorithm}

In order to demonstrate this step on the example population, assume that $\sm{3}$ has been selected for mutation and $\ro{2}$ is randomly selected from the rotation poset. Since $\ro{2} \not\in S_3$, $\ro{2}$ and all its predecessors are added to $S_3$, resulting in $\{\ro{0}, \ro{1}, \ro{4}\} \cup \{\ro{2}\} = S_6$. The new stable matching is $\sm{6}$ is added to the population and the mutated stable matching, $\sm{3}$ is deleted.
The final population is: $\langle \sm{6}, \sm{2}, \sm{4} \rangle$.

The procedure for the overall algorithm is given in Algorithm~\ref{alg:geneticalgorithm}.
The evolution step is repeated until either a solution is found or the termination criteria are met. In our case, we have two termination criteria: a time limit $lim_{time}$ and an iteration limit to cut-off the process if there is no improved solution for the specified number of iterations $lim_{cutoff}$. 
For cut-off termination criterion, we keep a counter $cnt$ to count the number of iterations with no  improvement, i.e. more fit solution than the current fittest solution.
This counter is increased at each unimproved iteration until it reaches the cut-off limit.
If a more fit solution is generated, $cnt$ is restarted as $cnt = 0$.

In order to compute the b values of the stable matchings, we use the method described in the Section~\ref{sec:verification} for each stable matching. 
Recall that this procedure takes 
$O(n \times |\roset|)$ time after a $O(n^2 + |\roset|^2)$ preprocessing step. 
In order to speed up this process from a practical point of view, in our experiments an extra data structure is used to memorize the fitness value of already generated stable matchings.
In both crossover and mutation, the worst case time complexity for one call
 is bounded by $O(| \roset |)$ since the set of all predecessors (respectively successors) of a given node is $\preT{\ro{}}$ (respectively $\sucT{\ro{}}$).

\begin{algorithm}
\caption{Genereal Procedure of Genetic Algorithm}\label{alg:geneticalgorithm}
\begin{algorithmic}[1]
\Procedure{GeneticAlgorithm}{$ $}

\State $\boldsymbol{P} \leftarrow$ \Call{Initialize}{$n$}
\State \Call{Evaluation}{$ $}
\State $\sm{fit} \leftarrow null$ 
\State $b_{fit} \leftarrow 0$ 
\State $cnt \leftarrow 0$ 
\While {$time < lim_{time}$ or $b_{fit} != 1$}

\State \Call{Evolution}{$ $}
\State \Call{Evaluation}{$ $}

 \If {$\sm{fit}$ is updated}
  \State $cnt \leftarrow 0$
 \Else
  \State $cnt \leftarrow cnt + 1$
 \EndIf
 
\If {$cnt = lim_{cutoff}$}
 \State return $\sm{fit}$
\EndIf

\EndWhile
\State return $\sm{fit}$

\EndProcedure
\end{algorithmic}
\end{algorithm}

\subsection{Local Search}

In our local search model, we use the well-known iterated local search with restarts ~\cite{stutzle1998local,localsearch}.
The first step of the algorithm is to find a random solution to start with. 
A random stable matching $M$ is created by using the Initialization method detailed in Algorithm~\ref{alg:initialization} in Section~\ref{meta:geneticAlg}.
Then, a neighbour stable match set is created using the properties of rotation poset. 

A \textit{neighbour} in this context is defined as a stable matching that differs only by one rotation from the closed subset $S$ of the matching $M$.
Let $L_s$ denote the set of leaf rotations of $S$.
Removing one rotation $\ro{i} \in L_s$ from $S$ corresponds to a different dominating neighbour of $M$. 
Similarly, let $N_s$ denote the set of rotations that are not included in $S$ and either they have an in-degree of 0 or all of their predecessors are in $S$.
In the same manner, adding a rotation from $N_s$ to $S$ at a time corresponds to a different dominated neighbour of $M$.

Figure~\ref{figure:localSearch} illustrates a sample rotation poset, where the closed subset is $S = \{\ro{0}, \ro{1}, \ro{3}, \ro{4}, \ro{5}, \ro{7}\}$. 
Then, the set of leaf rotations in this set is identified as $L_s = \{\ro{1}, \ro{4}, \ro{7}\}$. 
Similarly, the set of neighbour rotations is $N_s = \{\ro{2}, \ro{6}\}$ since all their predecessors are included in the closed subset $S$.

\begin{figure}[ht]
    \centering
 	\includegraphics[width=.35\textwidth]{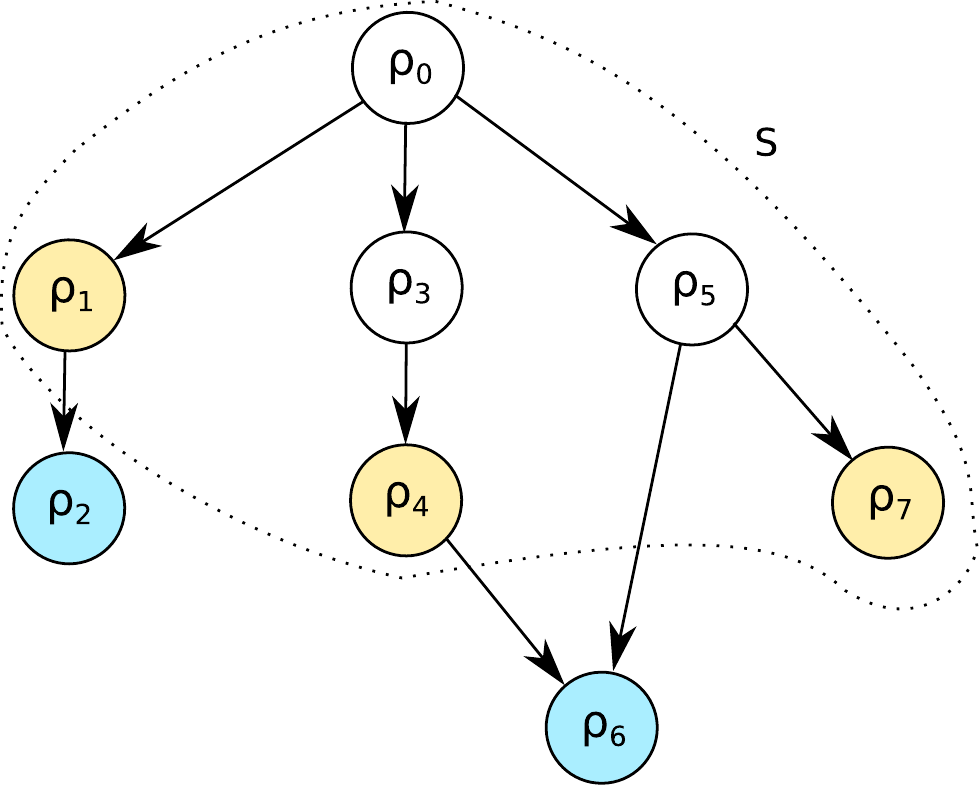}
    \caption{Illustration of the sets $L_s = \{ \ro{1}, \ro{4}, \ro{7} \}$ and $N_s = \{ \ro{2}, \ro{6} \}$ on a sample rotation poset for a given closed subset $S$.}
    \label{figure:localSearch}
\end{figure}

These two sets correspond to rotations, where removal of any rotation $\ro{} \in L_s$ from $S$ at a time does not require removal of any other rotations in order to obtain another closed subset, i.e. a neighbour. 
Likewise, adding any rotation $\ro{} \in N_s$ to $S$ does not require adding any additional rotations to obtain a closed subset.
Algorithm~\ref{alg:neighbours} expands on the procedure for identifying the set of neighbour stable matchings $N$ of a given stable matching $\sm{}$.
The lines between~\ref{lab:LSinStart} and~\ref{lab:LSinEnd} details the procedure of finding the rotations $\ro{}$ in $N_s$
Similarly, the lines between~\ref{lab:LSoutStart} and~\ref{lab:LSoutEnd} defines the procedure for identifying the leaf nodes $\ro{}$ in $L_s$.
The neighbour stable matchings are denoted by $\sm{n}$.
In this procedure, a variable called $cnt$ is used to count the number of incoming edges of $\ro{}$, where the source node is not in $S$ and then to count the number of outgoing edges whose target node is a member of $S$.

\begin{algorithm}
\caption{Identification of Neighbour Set}\label{alg:neighbours}
\begin{algorithmic}[1]
\Procedure{FindNeighbours}{$\sm{}$}

\State $N \leftarrow \emptyset$

\For{$\ro{} \in \roset$}
\label{lab:LSinStart}
 \If{$\ro{} \notin S$}
   \State $cnt \leftarrow 0$
   
   \For{ $e \in$ \Call{IncomingEdges}{$\ro{}$}}
      \If{$e.source \notin S$}
        \State $cnt \leftarrow cnt + 1$ 
      \EndIf
   \EndFor
   
   \If{$cnt = 0$} 
     \State $\sm{n} \leftarrow$ \Call{CreateSM}{$S \cup \{ \ro{} \}$}
     \State $N \leftarrow N \cup \{ \sm{n} \}$
     \label{lab:LSinEnd}
   \EndIf
   
 \EndIf
\EndFor

\For{$\ro{} \in S$}
\label{lab:LSoutStart}
   \State $cnt \leftarrow 0$
   
   \For{ $e \in$ \Call{OutgoingEdges}{$\ro{}$}}
      \If{$e.target \in S$}
        \State $cnt \leftarrow cnt + 1$
      \EndIf
   \EndFor
   
   \If{$cnt = 0$} 
     \State $\sm{n} \leftarrow$ \Call{CreateSM}{$S \setminus \{ \ro{} \}$}
     \State $N \leftarrow N \cup \{ \sm{n} \}$
     \label{lab:LSoutEnd}
   \EndIf
\EndFor

\Return $N$

\EndProcedure
\end{algorithmic}
\end{algorithm}

The general procedure is detailed in Algorithm~\ref{alg:localsearch}. 
At each iteration, if a neighbour has lower b than $M$, in other words, it is a more robust solution, the search continues by finding the neighbours of the new solution. 
A variable $\sm{best}$ is used to keep track of the best solution found so far in the whole search process.
We denote the best stable matching in set $N$ by $\sm{nb}$ and the function $\Call{Best}{$N$}$ is used to find the stable matching with the lowest b in set $N$.
Also a counter $cnt$ is used to count the number of iterations with no improved solutions.
The cut-off limit is bounded by a variable denoted by $lim_{cutoff}$ in the algorithm.
There is an iteration limit $lim_{depth}$ that indicates the depth of searching for neighbours of a randomly created stable matching.
After the iteration limit is met, a new random stable matching is generated and the neighbour search continues from that stable matching.

\begin{algorithm}
\caption{General Procedure for Local Search}\label{alg:localsearch}
\begin{algorithmic}[1]
\Procedure{LocalSearch}{}
\State $\sm{} \leftarrow $ \Call{CreateRandomSM}{$ $}
\State $\sm{best} \leftarrow \sm{}$
\State $cnt \leftarrow 0$

 \While {$time < lim_{time}$}
    \State $cnt \leftarrow 0$ 
    \State $iter_{depth} \leftarrow 0$ 
    \While {$iter_{depth} \leq lim_{depth}$}
        \State $cnt \leftarrow cnt + 1$
        \State $iter_{depth} \leftarrow iter_{depth} + 1$
        \State $N \leftarrow $\Call{FindNeighbors}{$\sm{}$}
        \State $\sm{nb} \leftarrow \Call{Best}{$N$}$
        \If {$\sm{nb}.b < \sm{}.b$}
            \State $\sm{} \leftarrow \sm{nb}$
            \If {$\sm{}.b < \sm{best}.b$}
            	\State $\sm{best} \leftarrow \sm{}$
           		\State $cnt \leftarrow 0$
            \EndIf
            
        \EndIf
        
       \If {$cnt = lim_{cutoff}$}
           \State return $\sm{best}$
       \EndIf
    \EndWhile
    
  	\State $\sm{} \leftarrow $ \Call{CreateRandomSM}{$ $}

\EndWhile
\State return $\sm{best}$
\EndProcedure
\end{algorithmic}
\end{algorithm}

Let us illustrate the neighbourhood search on the Stable Marriage instance given in Table~\ref{table:sm}. Assume that the search starts with a randomly generated stable matching, $\sm{5}$, where $S_5 = \{ \ro{0}, \ro{1}, \ro{2} \}$. The sets $L_{S_5}$ and $N_{S_5}$ are identified as follows: $L_{S_5} = \{ \ro{2} \}$, $N_{S_5} = \{ \ro{3}, \ro{4} \}$. Thus the current stable matching has 3 neighbours:\\
$S_5 \setminus \{ \ro{2} \} = \{ \ro{0}, \ro{1} \} = S_2$\\
$S_5 \cup \{ \ro{3} \} = \{ \ro{0}, \ro{1}, \ro{2}, \ro{3} \} = S_8$\\
$S_5 \cup \{ \ro{4} \} = \{ \ro{0}, \ro{1}, \ro{2}, \ro{4} \} = S_6$.

Next step is to calculate the robustness of each stable matching. Using the method described in Section~\ref{sec:verification}, the b values are calculated as follows: $\sm{5}.b = 3, \sm{2}.b = 3, \sm{8}.b = 3, \sm{6}.b = 1$. Since $\sm{6}$ is the most robust stable matching in the neighbourhood and better than the current stable matching $\sm{5}$, the neighbourhood search for the next step expands from $\sm{6}$.

	Notice that there can be at most $|\roset|$ neighbours of a stable matching. Thus, creating neighbours and finding their respective $b$ values is $O(n \times |\roset|^2)$.

\section{Experiments}
\label{sec:Experiments}

We experimentally evaluate the three models for finding the most robust stable matching. 
The CP model is implemented in Choco 4.0.1 ~\cite{chocoSolver} and the two meta-heuristics are implemented in Java. 
All experiments were performed on DELL M600 with 2.66 Ghz processors under Linux.



We ran each model with $4$ different randomization seeds for each instance. The time limit is fixed to 20 minutes for every run.   
An additional cut-off is used for local search (LS) and genetic algorithm (GA) as follows: 
if the solution quality does not improve for 10000 iterations, we terminate the search.
The genetic algorithm applies a crossover at each iteration unless the roulette wheel selection selects the fittest stable matching from the population. 
Additionally, the probability of applying mutation on a randomly selected stable matching is fixed as $80\%$. For local search, we chose to restart the local search with a randomly generated stable matching every 50 iterations. Last, the CP model (CP) uses the weighted degree heuristic~\cite{04-wdeg} with geometric restarts. 

We use two sets of random instances. 
The first set contains $500$ instances. The number of men for this set
is in the set $\{ 300 + 50*k \}$ where $k \in \{1,5\}$. 
The second set contains $600$ large instances of size $\{ 1200 + 50*k \}$ where $k \in \{1,6\}$. 
We generated $100$ instances for each size for both benchmarks.

In Figures~\ref{fig:optset1} and~\ref{fig:optset2} we plot the normalized objective value of the best solution found by the search model $h \in \{ {\rm CP, GA,LS} \}$ ($x$-axis) after a given time ($y$-axis). Let $h(I)$ be the objective value of the best solution found using model $h$ on instance $I$ and $lb(I)$ (resp. $ub(I)$) the lowest (resp. highest) objective value found by any model on $I$.
The formula below gives a normalized score in the interval $[0,1]$:
 
$$
    score(h,I)=
    \frac{ub(I) - h(I) + 1}{ub(I) - lb(I) + 1}
$$

The value of $score(h,I)$ is equal to $1$ if $h$ has found the best solution for this instance among all models, decreases as $h(I)$ gets further from the optimal objective value, and is equal to $0$ if and only if $h$ did not find any solution for $I$.

\begin{figure}[t]
	 \begin{center}
		\caption{\label{fig:optset1}Search Efficiency on the First Set of Instances}
 	\input{cactus_small_obj.tex}
	 \end{center}
\end{figure}

\begin{figure}[t]
	 \begin{center}
		\caption{\label{fig:optset2}Search Efficiency on Large Instances}
	\input{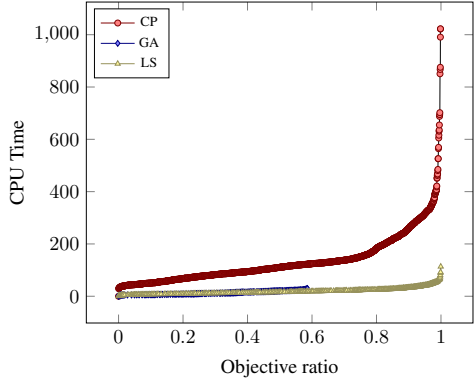}
	\end{center}
\end{figure}

Note that the CP model runs out of memory for large instances. Therefore, we do not plot it in Figure~\ref{fig:optset2}.



The outcome from both figures is clear. Local search is extremely efficient both in the quality of the solutions, and in runtime. Indeed, in the first plot, the best solutions are found by LS and CP. However, CP takes much longer time. Note that in Figure~\ref{fig:optset1}, CP and LS both find almost always the same objective, and in fact CP claims that the solution is optimal in all instances except one. 

The GA model does not seem to be well suited for this problem. 
In the first data set it does not find the best solutions in all instances. Moreover, it clearly takes much longer time than CP and LS for finding good quality solutions. 

The results in Figure~\ref{fig:optset2} are more spectacular for local search. 
Local search does not always find the best solutions since the normalised objective ratio does not exceed $90\%$. However, the overall performance is clearly better than GA in both quality and runtime.

\section{Conclusions}


We studied the notion of robustness in stable matching problems by using the notion of $(a,b)$-supermatch. 
We first showed that the problem of finding a stable matching $M_i$ that is closest to a given stable matching $M$ if a pair (man,woman) decides to break their match in $M$ can be found in polynomial time. 
Then, we used essentially this procedure to model the problem of finding the most robust stable matching using a CP formulation, local search, and genetic algorithm.
Last, we empirically evaluated these models on randomly generated instances and showed that local search is by far the best model to find robust solutions. 

To the best of our knowledge, this notion of robustness in stable matchings has never been proposed before. We hope that the proposed problem will get some attention in the future as it represents a challenge in real-world settings. 

\section*{Acknowledgments}

This publication has emanated from research conducted with the financial support of Science Foundation Ireland (SFI) under Grant Number SFI/12/RC/2289.

\appendix

\section{CP Model}
Below is an overview of the CP model for the Stable Marriage instance provided in Table~\ref{table:sm}.

As an example, find below the full CP model of the sample Stable Marriage instance given in Table~\ref{table:sm}.
The set of stable pairs, $SP = \{  \pair{0}{5}, $ $\pair{0}{2}, $ $\pair{0}{4}, $ $\pair{0}{1},$ $\pair{1}{4}, $ $\pair{1}{5}, $ $\pair{1}{3}, $ $\pair{2}{6}, $ $\pair{2}{0}, $ $\pair{3}{3}, $ $\pair{3}{5}, $ $\pair{4}{1}, $ $\pair{4}{4}, $ $\pair{5}{0}, $ $\pair{5}{4}, $ $\pair{6}{2}, $ $\pair{6}{5}, $ $\pair{6}{0}, $ $\pair{6}{6} \}$.

The constraints ensuring if a pair is a part of the solution by checking the existence of the rotations that produces and eliminates each pair:\\
$\x{0}{5} \leftrightarrow \neg{\s{0}}$\\
$\x{0}{2} \leftrightarrow \s{0} \wedge \neg{\s{2}}$\\
$\x{0}{4} \leftrightarrow \s{2} \wedge \neg{\s{3}}$\\
$\x{0}{1} \leftrightarrow \s{3}$\\
$\x{1}{4} \leftrightarrow \neg{\s{1}}$\\
$\x{1}{5} \leftrightarrow \s{1} \wedge \neg{\s{5}}$\\
$\x{1}{3} \leftrightarrow \s{5}$\\
$\x{2}{6} \leftrightarrow \neg{\s{4}}$\\
$\x{2}{0} \leftrightarrow \s{4}$\\
$\x{3}{3} \leftrightarrow \neg{\s{5}}$\\
$\x{3}{5} \leftrightarrow \s{5}$\\
$\x{4}{1} \leftrightarrow \neg{\s{3}}$\\
$\x{4}{4} \leftrightarrow \s{3}$\\
$\x{5}{0} \leftrightarrow \neg{\s{1}}$\\
$\x{5}{4} \leftrightarrow \s{1} \wedge \neg{\s{2}}$\\
$\x{5}{2} \leftrightarrow \s{2}$\\
$\x{6}{2} \leftrightarrow \neg{\s{0}}$\\
$\x{6}{5} \leftrightarrow \s{0} \wedge \neg{\s{1}}$\\
$\x{6}{0} \leftrightarrow \s{1} \wedge \neg{\s{4}}$\\
$\x{6}{6} \leftrightarrow \s{4}$\\

The constraints to make sure the solution forms a closed subset by requiring all predecessors of each rotation:\\
$\s{1} \rightarrow \s{0}$\\
$\s{2} \rightarrow \s{1}$\\
$\s{3} \rightarrow \s{2}$\\
$\s{4} \rightarrow \s{1}$\\
$\s{5} \rightarrow \s{4}$\\

The constraints that handle the case if a man is matched to his best possible partner:\\
$\alpha_{0} \leftrightarrow \x{0}{5}$\\
$\alpha_{1} \leftrightarrow \x{1}{4}$\\
$\alpha_{2} \leftrightarrow \x{2}{6}$\\
$\alpha_{3} \leftrightarrow \x{3}{3}$\\
$\alpha_{4} \leftrightarrow \x{4}{1}$\\
$\alpha_{5} \leftrightarrow \x{5}{0}$\\
$\alpha_{6} \leftrightarrow \x{6}{2}$\\

The constraints to state if a couple is no longer wanted, the rotation that produces the couple must be removed from the solution closed subset:\\
$\x{0}{2} \leftrightarrow \sUp{0}{0}$\\
$\x{0}{4} \leftrightarrow \sUp{2}{0}$\\
$\x{0}{1} \leftrightarrow \sUp{3}{0}$\\
$\x{1}{5} \leftrightarrow \sUp{1}{1}$\\
$\x{1}{3} \leftrightarrow \sUp{5}{1}$\\
$\x{2}{0} \leftrightarrow \sUp{4}{2}$\\
$\x{3}{5} \leftrightarrow \sUp{5}{3}$\\
$\x{4}{4} \leftrightarrow \sUp{3}{4}$\\
$\x{5}{4} \leftrightarrow \sUp{1}{5}$\\
$\x{5}{2} \leftrightarrow \sUp{2}{5}$\\
$\x{6}{5} \leftrightarrow \sUp{0}{6}$\\
$\x{6}{0} \leftrightarrow \sUp{1}{6}$\\
$\x{6}{6} \leftrightarrow \sUp{4}{6}$\\

The constraints that ensure there does not exist a repair matching $\sManUp{i}$ for man $i$ if he is matched to his best possible partner:\\
$\alpha_{0} \rightarrow \neg{\sUp{0}{0}} \wedge \neg{\sUp{1}{0}} \wedge \neg{\sUp{2}{0}} \wedge \neg{\sUp{3}{0}} \wedge \neg{\sUp{4}{0}} \wedge \neg{\sUp{5}{0}}$\\
$\alpha_{1} \rightarrow \neg{\sUp{0}{1}} \wedge \neg{\sUp{1}{1}} \wedge \neg{\sUp{2}{1}} \wedge \neg{\sUp{3}{1}} \wedge \neg{\sUp{4}{1}} \wedge \neg{\sUp{5}{1}}$\\
$\alpha_{2} \rightarrow \neg{\sUp{0}{2}} \wedge \neg{\sUp{1}{2}} \wedge \neg{\sUp{2}{2}} \wedge \neg{\sUp{3}{2}} \wedge \neg{\sUp{4}{2}} \wedge \neg{\sUp{5}{2}}$\\
$\alpha_{3} \rightarrow \neg{\sUp{0}{3}} \wedge \neg{\sUp{1}{3}} \wedge \neg{\sUp{2}{3}} \wedge \neg{\sUp{3}{3}} \wedge \neg{\sUp{4}{3}} \wedge \neg{\sUp{5}{3}}$\\
$\alpha_{4} \rightarrow \neg{\sUp{0}{4}} \wedge \neg{\sUp{1}{4}} \wedge \neg{\sUp{2}{4}} \wedge \neg{\sUp{3}{4}} \wedge \neg{\sUp{4}{4}} \wedge \neg{\sUp{5}{4}}$\\
$\alpha_{5} \rightarrow \neg{\sUp{0}{5}} \wedge \neg{\sUp{1}{5}} \wedge \neg{\sUp{2}{5}} \wedge \neg{\sUp{3}{5}} \wedge \neg{\sUp{4}{5}} \wedge \neg{\sUp{5}{5}}$\\
$\alpha_{6} \rightarrow \neg{\sUp{0}{6}} \wedge \neg{\sUp{1}{6}} \wedge \neg{\sUp{2}{6}} \wedge \neg{\sUp{3}{6}} \wedge \neg{\sUp{4}{6}} \wedge \neg{\sUp{5}{6}}$\\

The constraints that require if a rotation is in the difference set $S \setminus \sManUp{i}$, then all successors of that rotation that are a part of solution is also in the difference set when $i = 0$:
$\sUp{0}{0} \wedge \s{1} \rightarrow \sUp{1}{0}$\\
$\sUp{0}{0} \wedge \s{2} \rightarrow \sUp{2}{0}$\\
$\sUp{0}{0} \wedge \s{3} \rightarrow \sUp{3}{0}$\\
$\sUp{0}{0} \wedge \s{4} \rightarrow \sUp{4}{0}$\\
$\sUp{0}{0} \wedge \s{5} \rightarrow \sUp{5}{0}$\\
$\sUp{1}{0} \wedge \s{2} \rightarrow \sUp{2}{0}$\\
$\sUp{1}{0} \wedge \s{3} \rightarrow \sUp{3}{0}$\\
$\sUp{1}{0} \wedge \s{4} \rightarrow \sUp{4}{0}$\\
$\sUp{1}{0} \wedge \s{5} \rightarrow \sUp{5}{0}$\\
$\sUp{2}{0} \wedge \s{3} \rightarrow \sUp{3}{0}$\\
$\sUp{4}{0} \wedge \s{5} \rightarrow \sUp{5}{0}$\\

when $i = 1$\\
$\sUp{0}{1} \wedge \s{1} \rightarrow \sUp{1}{1}$\\
$\sUp{0}{1} \wedge \s{2} \rightarrow \sUp{2}{1}$\\
$\sUp{0}{1} \wedge \s{3} \rightarrow \sUp{3}{1}$\\
$\sUp{0}{1} \wedge \s{4} \rightarrow \sUp{4}{1}$\\
$\sUp{0}{1} \wedge \s{5} \rightarrow \sUp{5}{1}$\\
$\sUp{1}{1} \wedge \s{2} \rightarrow \sUp{2}{1}$\\
$\sUp{1}{1} \wedge \s{3} \rightarrow \sUp{3}{1}$\\
$\sUp{1}{1} \wedge \s{4} \rightarrow \sUp{4}{1}$\\
$\sUp{1}{1} \wedge \s{5} \rightarrow \sUp{5}{1}$\\
$\sUp{2}{1} \wedge \s{3} \rightarrow \sUp{3}{1}$\\
$\sUp{4}{1} \wedge \s{5} \rightarrow \sUp{5}{1}$\\

when $i = 2$\\
$\sUp{0}{2} \wedge \s{1} \rightarrow \sUp{1}{2}$\\
$\sUp{0}{2} \wedge \s{2} \rightarrow \sUp{2}{2}$\\
$\sUp{0}{2} \wedge \s{3} \rightarrow \sUp{3}{2}$\\
$\sUp{0}{2} \wedge \s{4} \rightarrow \sUp{4}{2}$\\
$\sUp{0}{2} \wedge \s{5} \rightarrow \sUp{5}{2}$\\
$\sUp{1}{2} \wedge \s{2} \rightarrow \sUp{2}{2}$\\
$\sUp{1}{2} \wedge \s{3} \rightarrow \sUp{3}{2}$\\
$\sUp{1}{2} \wedge \s{4} \rightarrow \sUp{4}{2}$\\
$\sUp{1}{2} \wedge \s{5} \rightarrow \sUp{5}{2}$\\
$\sUp{2}{2} \wedge \s{3} \rightarrow \sUp{3}{2}$\\
$\sUp{4}{2} \wedge \s{5} \rightarrow \sUp{5}{2}$\\

when $i = 3$\\
$\sUp{0}{3} \wedge \s{1} \rightarrow \sUp{1}{3}$\\
$\sUp{0}{3} \wedge \s{2} \rightarrow \sUp{2}{3}$\\
$\sUp{0}{3} \wedge \s{3} \rightarrow \sUp{3}{3}$\\
$\sUp{0}{3} \wedge \s{4} \rightarrow \sUp{4}{3}$\\
$\sUp{0}{3} \wedge \s{5} \rightarrow \sUp{5}{3}$\\
$\sUp{1}{3} \wedge \s{2} \rightarrow \sUp{2}{3}$\\
$\sUp{1}{3} \wedge \s{3} \rightarrow \sUp{3}{3}$\\
$\sUp{1}{3} \wedge \s{4} \rightarrow \sUp{4}{3}$\\
$\sUp{1}{3} \wedge \s{5} \rightarrow \sUp{5}{3}$\\
$\sUp{2}{3} \wedge \s{3} \rightarrow \sUp{3}{3}$\\
$\sUp{4}{3} \wedge \s{5} \rightarrow \sUp{5}{3}$\\

when $i = 4$\\
$\sUp{0}{4} \wedge \s{1} \rightarrow \sUp{1}{4}$\\
$\sUp{0}{4} \wedge \s{2} \rightarrow \sUp{2}{4}$\\
$\sUp{0}{4} \wedge \s{3} \rightarrow \sUp{3}{4}$\\
$\sUp{0}{4} \wedge \s{4} \rightarrow \sUp{4}{4}$\\
$\sUp{0}{4} \wedge \s{5} \rightarrow \sUp{5}{4}$\\
$\sUp{1}{4} \wedge \s{2} \rightarrow \sUp{2}{4}$\\
$\sUp{1}{4} \wedge \s{3} \rightarrow \sUp{3}{4}$\\
$\sUp{1}{4} \wedge \s{4} \rightarrow \sUp{4}{4}$\\
$\sUp{1}{4} \wedge \s{5} \rightarrow \sUp{5}{4}$\\
$\sUp{2}{4} \wedge \s{3} \rightarrow \sUp{3}{4}$\\
$\sUp{4}{4} \wedge \s{5} \rightarrow \sUp{5}{4}$\\

when $i = 5$\\
$\sUp{0}{5} \wedge \s{1} \rightarrow \sUp{1}{5}$\\
$\sUp{0}{5} \wedge \s{2} \rightarrow \sUp{2}{5}$\\
$\sUp{0}{5} \wedge \s{3} \rightarrow \sUp{3}{5}$\\
$\sUp{0}{5} \wedge \s{4} \rightarrow \sUp{4}{5}$\\
$\sUp{0}{5} \wedge \s{5} \rightarrow \sUp{5}{5}$\\
$\sUp{1}{5} \wedge \s{2} \rightarrow \sUp{2}{5}$\\
$\sUp{1}{5} \wedge \s{3} \rightarrow \sUp{3}{5}$\\
$\sUp{1}{5} \wedge \s{4} \rightarrow \sUp{4}{5}$\\
$\sUp{1}{5} \wedge \s{5} \rightarrow \sUp{5}{5}$\\
$\sUp{2}{5} \wedge \s{3} \rightarrow \sUp{3}{5}$\\
$\sUp{4}{5} \wedge \s{5} \rightarrow \sUp{5}{5}$\\

when $i = 6$\\
$\sUp{0}{6} \wedge \s{1} \rightarrow \sUp{1}{6}$\\
$\sUp{0}{6} \wedge \s{2} \rightarrow \sUp{2}{6}$\\
$\sUp{0}{6} \wedge \s{3} \rightarrow \sUp{3}{6}$\\
$\sUp{0}{6} \wedge \s{4} \rightarrow \sUp{4}{6}$\\
$\sUp{0}{6} \wedge \s{5} \rightarrow \sUp{5}{6}$\\
$\sUp{1}{6} \wedge \s{2} \rightarrow \sUp{2}{6}$\\
$\sUp{1}{6} \wedge \s{3} \rightarrow \sUp{3}{6}$\\
$\sUp{1}{6} \wedge \s{4} \rightarrow \sUp{4}{6}$\\
$\sUp{1}{6} \wedge \s{5} \rightarrow \sUp{5}{6}$\\
$\sUp{2}{6} \wedge \s{3} \rightarrow \sUp{3}{6}$\\
$\sUp{4}{6} \wedge \s{5} \rightarrow \sUp{5}{6}$\\

The constraints that ensure if a rotation is in the difference set $S \setminus \sManUp{i}$, they must be a part of the solution and they are either the rotation that produces the couple in the solution or they have a predecessor that produces the couple. For $v = \ro{0}$:\\
$\sUp{0}{0} \rightarrow \s{0} \wedge \x{0}{2}$\\
$\sUp{0}{1} \rightarrow \s{0}$\\
$\sUp{0}{2} \rightarrow \s{0}$\\
$\sUp{0}{3} \rightarrow \s{0}$\\
$\sUp{0}{4} \rightarrow \s{0}$\\
$\sUp{0}{5} \rightarrow \s{0}$\\
$\sUp{0}{6} \rightarrow \s{0} \wedge \x{6}{5}$\\

For $v = \ro{1}$:\\
$\sUp{1}{0} \rightarrow \s{1} \wedge \sUp{0}{0}$\\
$\sUp{1}{1} \rightarrow \s{1} \wedge (\x{1}{5} \vee \sUp{0}{1}) $\\
$\sUp{1}{2} \rightarrow \s{1} \wedge \sUp{0}{2}$\\
$\sUp{1}{3} \rightarrow \s{1} \wedge \sUp{0}{3}$\\
$\sUp{1}{4} \rightarrow \s{1} \wedge \sUp{0}{4}$\\
$\sUp{1}{5} \rightarrow \s{1} \wedge (\x{5}{4} \vee \sUp{0}{5})$\\
$\sUp{1}{6} \rightarrow \s{1} \wedge (\x{6}{0} \vee \sUp{0}{6})$\\

For $v = \ro{2}$:\\
$\sUp{2}{0} \rightarrow \s{2} \wedge (\x{0}{4} \vee \sUp{0}{0} \vee \sUp{1}{0})$\\
$\sUp{2}{1} \rightarrow \s{2} \wedge (\sUp{0}{1} \vee \sUp{1}{1}) $\\
$\sUp{2}{2} \rightarrow \s{2} \wedge (\sUp{0}{2} \vee \sUp{1}{2})$\\
$\sUp{2}{3} \rightarrow \s{2} \wedge (\sUp{0}{3} \vee \sUp{1}{3})$\\
$\sUp{2}{4} \rightarrow \s{2} \wedge (\sUp{0}{4} \vee \sUp{1}{4})$\\
$\sUp{2}{5} \rightarrow \s{2} \wedge (\x{5}{4} \vee \sUp{0}{5} \vee \sUp{1}{5})$\\
$\sUp{2}{6} \rightarrow \s{2} \wedge (\sUp{0}{6} \vee \sUp{1}{6})$\\

For $v = \ro{3}$:\\
$\sUp{3}{0} \rightarrow \s{3} \wedge (\x{0}{1} \vee \sUp{0}{0} \vee \sUp{1}{0} \vee \sUp{2}{0})$\\
$\sUp{3}{1} \rightarrow \s{3} \wedge (\sUp{0}{1} \vee \sUp{1}{1} \vee \sUp{2}{1}) $\\
$\sUp{3}{2} \rightarrow \s{3} \wedge (\sUp{0}{2} \vee \sUp{1}{2} \vee \sUp{2}{2})$\\
$\sUp{3}{3} \rightarrow \s{3} \wedge (\sUp{0}{3} \vee \sUp{1}{3} \vee \sUp{2}{3})$\\
$\sUp{3}{4} \rightarrow \s{3} \wedge (\x{4}{4}  \vee \sUp{0}{4} \vee \sUp{1}{4} \vee \sUp{2}{4})$\\
$\sUp{3}{5} \rightarrow \s{3} \wedge (\sUp{0}{5} \vee \sUp{1}{5} \vee \sUp{2}{5})$\\
$\sUp{3}{6} \rightarrow \s{3} \wedge (\sUp{0}{6} \vee \sUp{1}{6}  \vee \sUp{2}{6})$\\

For $v = \ro{4}$:\\
$\sUp{4}{0} \rightarrow \s{4} \wedge (\sUp{0}{0} \vee \sUp{1}{0})$\\
$\sUp{4}{1} \rightarrow \s{4} \wedge (\sUp{0}{1} \vee \sUp{1}{1}) $\\
$\sUp{4}{2} \rightarrow \s{4} \wedge (\x{2}{0} \vee \sUp{0}{2} \vee \sUp{1}{2})$\\
$\sUp{4}{3} \rightarrow \s{4} \wedge (\sUp{0}{3} \vee \sUp{1}{3})$\\
$\sUp{4}{4} \rightarrow \s{4} \wedge (\sUp{0}{4} \vee \sUp{1}{4})$\\
$\sUp{4}{5} \rightarrow \s{4} \wedge (\sUp{0}{5} \vee \sUp{1}{5})$\\
$\sUp{4}{6} \rightarrow \s{4} \wedge (\x{6}{6} \vee \sUp{0}{6} \vee \sUp{1}{6})$\\

For $v = \ro{5}$:\\
$\sUp{5}{0} \rightarrow \s{5} \wedge (\sUp{0}{0} \vee \sUp{1}{0} \vee \sUp{4}{0})$\\
$\sUp{5}{1} \rightarrow \s{5} \wedge (\x{1}{3} \vee \sUp{0}{1} \vee \sUp{1}{1} \vee \sUp{4}{1}) $\\
$\sUp{5}{2} \rightarrow \s{5} \wedge (\sUp{0}{2} \vee \sUp{1}{2} \vee \sUp{4}{2})$\\
$\sUp{5}{3} \rightarrow \s{5} \wedge (\x{3}{5} \vee \sUp{0}{3} \vee \sUp{1}{3} \vee \sUp{4}{3})$\\
$\sUp{5}{4} \rightarrow \s{5} \wedge (\sUp{0}{4} \vee \sUp{1}{4} \vee \sUp{4}{4})$\\
$\sUp{5}{5} \rightarrow \s{5} \wedge (\sUp{0}{5} \vee \sUp{1}{5} \vee \sUp{4}{5})$\\
$\sUp{5}{6} \rightarrow \s{5} \wedge (\sUp{0}{6} \vee \sUp{1}{6}  \vee \sUp{4}{6})$\\

The constraints for construction of the difference set $\sManDown{i} \setminus S$ can also be generated similar to the difference set constraints above. We do not include those constraints here for space limitations.

For the counting constraints, following are for counting which other men are required to change their couples in order to repair man $i$ for difference set $S \setminus \sManUp{i}$, when $i = 0$:\\
$\y{1}{0} \leftrightarrow \sUp{1}{0} \vee \sUp{5}{0}$\\
$\y{2}{0} \leftrightarrow \sUp{4}{0}$\\
$\y{3}{0} \leftrightarrow \sUp{5}{0}$\\
$\y{4}{0} \leftrightarrow \sUp{3}{0}$\\
$\y{5}{0} \leftrightarrow \sUp{1}{0} \vee \sUp{2}{0}$\\
$\y{6}{0} \leftrightarrow \sUp{0}{0} \vee \sUp{1}{0} \vee \sUp{4}{0}$\\

when $i = 1$:\\
$\y{0}{1} \leftrightarrow \sUp{0}{1} \vee \sUp{2}{1} \vee \sUp{3}{1}$\\
$\y{2}{1} \leftrightarrow \sUp{4}{1}$\\
$\y{3}{1} \leftrightarrow \sUp{5}{1}$\\
$\y{4}{1} \leftrightarrow \sUp{3}{1}$\\
$\y{5}{1} \leftrightarrow \sUp{1}{1} \vee \sUp{2}{1}$\\
$\y{6}{1} \leftrightarrow \sUp{0}{1} \vee \sUp{1}{1} \vee \sUp{4}{1}$\\

when $i = 2$:\\
$\y{0}{2} \leftrightarrow \sUp{0}{2} \vee \sUp{2}{2} \vee \sUp{3}{2}$\\
$\y{1}{2} \leftrightarrow \sUp{1}{2} \vee \sUp{5}{2}$\\
$\y{3}{2} \leftrightarrow \sUp{5}{2}$\\
$\y{4}{2} \leftrightarrow \sUp{3}{2}$\\
$\y{5}{2} \leftrightarrow \sUp{1}{2} \vee \sUp{2}{2}$\\
$\y{6}{2} \leftrightarrow \sUp{0}{2} \vee \sUp{1}{2} \vee \sUp{4}{2}$\\

when $i = 3$:\\
$\y{0}{3} \leftrightarrow \sUp{0}{3} \vee \sUp{2}{3} \vee \sUp{3}{3}$\\
$\y{1}{3} \leftrightarrow \sUp{1}{3} \vee \sUp{5}{3}$\\
$\y{2}{3} \leftrightarrow \sUp{4}{3}$\\
$\y{4}{3} \leftrightarrow \sUp{3}{3}$\\
$\y{5}{3} \leftrightarrow \sUp{1}{3} \vee \sUp{2}{3}$\\
$\y{6}{3} \leftrightarrow \sUp{0}{3} \vee \sUp{1}{3} \vee \sUp{4}{3}$\\

when $i = 4$:\\
$\y{0}{4} \leftrightarrow \sUp{0}{4} \vee \sUp{2}{4} \vee \sUp{3}{4}$\\
$\y{1}{4} \leftrightarrow \sUp{1}{4} \vee \sUp{5}{4}$\\
$\y{2}{4} \leftrightarrow \sUp{4}{4}$\\
$\y{3}{4} \leftrightarrow \sUp{5}{4}$\\
$\y{5}{4} \leftrightarrow \sUp{1}{4} \vee \sUp{2}{4}$\\
$\y{6}{4} \leftrightarrow \sUp{0}{4} \vee \sUp{1}{4} \vee \sUp{4}{4}$\\

when $i = 5$:\\
$\y{0}{5} \leftrightarrow \sUp{0}{5} \vee \sUp{2}{5} \vee \sUp{3}{5}$\\
$\y{1}{5} \leftrightarrow \sUp{1}{5} \vee \sUp{5}{5}$\\
$\y{2}{5} \leftrightarrow \sUp{4}{5}$\\
$\y{3}{5} \leftrightarrow \sUp{5}{5}$\\
$\y{4}{5} \leftrightarrow \sUp{3}{5}$\\
$\y{6}{5} \leftrightarrow \sUp{0}{5} \vee \sUp{1}{5} \vee \sUp{4}{5}$\\

when $i = 6$:\\
$\y{0}{6} \leftrightarrow \sUp{0}{6} \vee \sUp{2}{6} \vee \sUp{3}{6}$\\
$\y{1}{6} \leftrightarrow \sUp{1}{6} \vee \sUp{5}{6}$\\
$\y{2}{6} \leftrightarrow \sUp{4}{6}$\\
$\y{3}{6} \leftrightarrow \sUp{5}{6}$\\
$\y{4}{6} \leftrightarrow \sUp{3}{6}$\\
$\y{5}{6} \leftrightarrow \sUp{1}{6} \vee \sUp{2}{6}$\\

For difference set $\sManDown{i} \setminus S$, the constraints are the same except the boolean variables $\y{}{}$ is replaced with $\z{}{}$ and $\sUp{}{}$ is replaced with $\sDown{}{}$.

The constraints for counting how many other men are required to change their partners in order to provide a repair stable matching that dominates the current matching:\\
$\alpha_0 =$ true $\rightarrow \dup{0} = n$, 
else $\rightarrow \dup{0} = \y{1}{0} + \y{2}{0} + \y{3}{0} + \y{4}{0} + \y{5}{0} + \y{6}{0}$\\
$\alpha_1 =$ true $\rightarrow \dup{1} = n$, 
else $\rightarrow \dup{1} = \y{0}{1} + \y{2}{1} + \y{3}{1} + \y{4}{1} + \y{5}{1} + \y{6}{1}$\\
$\alpha_2 =$ true $\rightarrow \dup{2} = n$, 
else $\rightarrow \dup{2} = \y{0}{2} + \y{1}{2} + \y{3}{2} + \y{4}{2} + \y{5}{2} + \y{6}{2}$\\
$\alpha_3 =$ true $\rightarrow \dup{3} = n$, 
else $\rightarrow \dup{3} = \y{0}{3} + \y{1}{3} + \y{2}{3} + \y{4}{3} + \y{5}{3} + \y{6}{3}$\\
$\alpha_4 =$ true $\rightarrow \dup{4} = n$, 
else $\rightarrow \dup{4} = \y{0}{4} + \y{1}{4} + \y{2}{4} + \y{3}{4} + \y{5}{4} + \y{6}{4}$\\
$\alpha_5 =$ true $\rightarrow \dup{5} = n$, 
else $\rightarrow \dup{5} = \y{0}{5} + \y{1}{5} + \y{2}{5} + \y{3}{5} + \y{4}{5} + \y{6}{5}$\\
$\alpha_6 =$ true $\rightarrow \dup{6} = n$, 
else $\rightarrow \dup{6} = \y{0}{6} + \y{1}{6} + \y{2}{6} + \y{3}{6} + \y{4}{6} + \y{5}{6}$\\

Similarly, the repairs required for the dominated stable matching:
$\beta_0 =$ true $\rightarrow \ddown{0} = n$, 
else $\rightarrow \ddown{0} = \z{1}{0} + \z{2}{0} + \z{3}{0} + \z{4}{0} + \z{5}{0} + \z{6}{0}$\\
$\beta_1 =$ true $\rightarrow \ddown{1} = n$, 
else $\rightarrow \ddown{1} = \z{0}{1} + \z{2}{1} + \z{3}{1} + \z{4}{1} + \z{5}{1} + \z{6}{1}$\\
$\beta_2 =$ true $\rightarrow \ddown{2} = n$, 
else $\rightarrow \ddown{2} = \z{0}{2} + \z{1}{2} + \z{3}{2} + \z{4}{2} + \z{5}{2} + \z{6}{2}$\\
$\beta_3 =$ true $\rightarrow \ddown{3} = n$, 
else $\rightarrow \ddown{3} = \z{0}{3} + \z{1}{3} + \z{2}{3} + \z{4}{3} + \z{5}{3} + \z{6}{3}$\\
$\beta_4 =$ true $\rightarrow \ddown{4} = n$, 
else $\rightarrow \ddown{4} = \z{0}{4} + \z{1}{4} + \z{2}{4} + \z{3}{4} + \z{5}{4} + \z{6}{4}$\\
$\beta_5 =$ true $\rightarrow \ddown{5} = n$, 
else $\rightarrow \ddown{5} = \z{0}{5} + \z{1}{5} + \z{2}{5} + \z{3}{5} + \z{4}{5} + \z{6}{5}$\\
$\beta_6 =$ true $\rightarrow \ddown{6} = n$, 
else $\rightarrow \ddown{6} = \z{0}{6} + \z{1}{6} + \z{2}{6} + \z{3}{6} + \z{4}{6} + \z{5}{6}$\\

The constraint for measuring the robustness of the solution:\\
$\big( \neg(\alpha_0) \rightarrow ( b \geq \dup{0}) \big) \vee \big( \neg(\beta_0) \rightarrow ( b \geq \ddown{0}) \big)$\\
$\big( \neg(\alpha_1) \rightarrow ( b \geq \dup{1}) \big) \vee \big( \neg(\beta_1) \rightarrow ( b \geq \ddown{1}) \big)$\\
$\big( \neg(\alpha_2) \rightarrow ( b \geq \dup{2}) \big) \vee \big( \neg(\beta_2) \rightarrow ( b \geq \ddown{2}) \big)$\\
$\big( \neg(\alpha_3) \rightarrow ( b \geq \dup{3}) \big) \vee \big( \neg(\beta_3) \rightarrow ( b \geq \ddown{3}) \big)$\\
$\big( \neg(\alpha_4) \rightarrow ( b \geq \dup{4}) \big) \vee \big( \neg(\beta_4) \rightarrow ( b \geq \ddown{4}) \big)$\\
$\big( \neg(\alpha_5) \rightarrow ( b \geq \dup{5}) \big) \vee \big( \neg(\beta_5) \rightarrow ( b \geq \ddown{5}) \big)$\\
$\big( \neg(\alpha_6) \rightarrow ( b \geq \dup{6}) \big) \vee \big( \neg(\beta_6) \rightarrow ( b \geq \ddown{6}) \big)$\\


\end{document}